\documentclass[english]{article}

\usepackage[T1]{fontenc}
\usepackage[utf8]{inputenc}
\usepackage{amsthm}
\usepackage{amsmath}
\usepackage{amssymb}
\usepackage{graphicx}
\usepackage{dsfont}
\usepackage{amsmath}
\usepackage{array}
\usepackage{multirow}
\usepackage{caption}
\usepackage{color}

\usepackage{multicol}

\theoremstyle{plain}

\theoremstyle{plain}

\theoremstyle{plain}
\newtheorem{lem}{\protect\lemmaname}
\theoremstyle{plain}
\newtheorem{thm}{\protect\theoremname}
\theoremstyle{plain}
  
\theoremstyle{definition}

\theoremstyle{definition}

\theoremstyle{definition}
\newtheorem{rem}{\protect\remarkname}

\makeatother
  
\usepackage{babel} 

\providecommand{\claimname}{Claim}
\providecommand{\lemmaname}{Lemma}
\providecommand{\propositionname}{Proposition}
\providecommand{\theoremname}{Theorem}
\providecommand{\corollaryname}{Corollary} 
\providecommand{\definitionname}{Definition}
\providecommand{\assumptionname}{Assumption}
\providecommand{\remarkname}{Remark}

\newcommand{\overbar}[1]{\mkern 1.25mu\overline{\mkern-1.25mu#1\mkern-0.25mu}\mkern 0.25mu}

\DeclareMathOperator*{\argmax}{arg\,max}

\newcommand{\openone}{\mathds{1}}

\newcommand{\Dbar}{\overbar{D}}
\newcommand{\Rbar}{\overbar{R}}
\newcommand{\Runder}{\underline{R}}
\newcommand{\bdelta}{\boldsymbol{\delta}}
\newcommand{\kSE}{k_{\text{SE}}}
\newcommand{\kMat}{k_{\text{Mat\'ern}}}
\newcommand{\ytil}{\tilde{y}}
\newcommand{\Otil}{\tilde{O}}

\newcommand{\vbar}{\overline{v}}

\newcommand{\Hc}{\mathcal{H}}

\newcommand{\xv}{\mathbf{x}}

\newcommand{\yv}{\mathbf{y}}

\newcommand{\zv}{\mathbf{z}}

\newcommand{\Ac}{\mathcal{A}}

\newcommand{\Fc}{\mathcal{F}}

\newcommand{\Ic}{\mathcal{I}}

\newcommand{\EE}{\mathbb{E}}
\newcommand{\PP}{\mathbb{P}}
\newcommand{\RR}{\mathbb{R}}
\newcommand{\ZZ}{\mathbb{Z}}

\newcommand{\Rc}{\mathcal{R}}

\newcommand{\Iv}{\mathbf{I}}

\newcommand{\Kv}{\mathbf{K}}

\newcommand{\bzero}{\boldsymbol{0}}

\newcommand{\bxi}{\boldsymbol{\xi}}

\usepackage{enumitem}

\usepackage{hyperref}
\usepackage{footnote}

\usepackage[accepted]{icml2021}

\newcommand{\dist}{{\rm dist}}
\renewcommand{\Otil}{O^{*}}

\icmltitlerunning{On Lower Bounds for Standard and Robust Gaussian Process Bandit Optimization}

\begin{document}

\twocolumn[
\icmltitle{On Lower Bounds for Standard and Robust \\ 
            Gaussian Process Bandit Optimization}




\begin{icmlauthorlist}
    \icmlauthor{Xu Cai}{1}
    \icmlauthor{Jonathan Scarlett}{1,2}
\end{icmlauthorlist}

\icmlaffiliation{1}{Department of Computer Science, National University of Singapore}
\icmlaffiliation{2}{Department of Mathematics \& Institute of Data Science, National University of Singapore}

\icmlcorrespondingauthor{Xu Cai}{caix@u.nus.edu}
\icmlcorrespondingauthor{Jonathan Scarlett}{scarlett@comp.nus.edu.sg}

\icmlkeywords{Bandits; Information theory}

\vskip 0.3in
]


\printAffiliationsAndNotice{}  

\begin{abstract}
    In this paper, we consider algorithm-independent lower bounds for the problem of black-box optimization of functions having a bounded norm is some Reproducing Kernel Hilbert Space (RKHS), which can be viewed as a non-Bayesian Gaussian process bandit problem.  In the standard noisy setting, we provide a novel proof technique for deriving lower bounds on the regret, with benefits including simplicity, versatility, and an improved dependence on the error probability.  In a robust setting in which every sampled point may be perturbed by a suitably-constrained adversary, we provide a novel lower bound for deterministic strategies, demonstrating an inevitable joint dependence of the cumulative regret on the corruption level and the time horizon, in contrast with existing lower bounds that only characterize the individual dependencies. Furthermore, in a distinct robust setting in which the final point is perturbed by an adversary, we strengthen an existing lower bound that only holds for target success probabilities very close to one, by allowing for arbitrary success probabilities above $\frac{2}{3}$.
\end{abstract}
\section{Introduction} \label{sec:intro}

The use of Gaussian process (GP) methods for black-box function optimization has seen significant advances in recent years, with applications including hyperparameter tuning, robotics, molecular design, and many more.  On the theoretical side, a variety of algorithms have been developed with provable regret bounds \cite{Sri09,Bul11,Con13,Wan16,Bog16a,Wan17,Jan20}, and algorithm-independent lower bounds have been given in several settings of interest \cite{Bul11,Sca17a,Sca18a,Cho19,Wan20}.

These theoretical works can be broadly categorized into one of two types: In the {\em Bayesian setting}, one adopts a Gaussian process prior according to some kernel function, whereas in  the {\em non-Bayesian setting}, the function is assumed to lie in some Reproducing Kernel Hilbert Space (RKHS) and be upper bounded in terms of the corresponding RKHS norm.  

In this paper, we focus on the non-Bayesian setting, and seek to broaden the existing understanding of algorithm-independent lower bounds on the regret, which have received significantly less attention than upper bounds.  Our main contributions are briefly summarized as follows:
\begin{itemize}[leftmargin=5ex,itemsep=0ex,topsep=0.2ex]
    \item In the standard noisy GP optimization setting, we provide an alternative proof strategy for the existing lower bounds of \cite{Sca17a}, which we believe to be of significant importance in itself due to the lack of techniques in the literature. We additionally show that our approach strengthens the dependence on the error probability, and give scenarios in which our approach is simpler and/or more versatile.
    \item We provide a novel lower bound for a robust setting in which the sampled points are adversarially corrupted \cite{Bog20}.  Our bound demonstrates that the cumulative regret of any deterministic algorithm must incur a certain {\em joint} dependence on the corruption level and time horizon, strengthening results from \cite{Bog20} stating that certain {\em separate} dependencies are unavoidable.
    \item We provide an improvement on an existing lower bound for a distinct robust setting \cite{Bog18}, in which the {\em final point returned} is perturbed by an adversary.  While the lower bound of \cite{Bog18} shows that a certain number of samples is needed to attain a certain level of regret with probability very close to one, we show that the same number of samples (up to constant factors) is required just to succeed with probability at least $\frac{2}{3}$.
\end{itemize}
The relevant existing results are highlighted throughout the paper, with further details in Appendix \ref{app:existing}.

\section{Problem Setup}

The three problem settings considered throughout the paper are formally described as follows.  We additionally informally summarize the existing lower bounds in each of these settings, with formal statements given in Appendix \ref{app:existing} along with existing upper bounds.  The existing and new bounds are summarized in Table \ref{tbl:standard} below.

\newcounter{auxFootnote}
\newcounter{auxFootnote2}

\begin{table*}
    
    {\centering \bf \underline{SE kernel} \par} 
    \smallskip
    \begin{centering}
        \begin{tabular}{|>{\centering}m{3.6cm}|>{\centering}m{4cm}|>{\centering}m{4cm}|>{\centering}m{3.8cm}|}
            
            \hline 
            & {\small \bf Upper Bound} & {\small \bf Existing Lower Bound } & {\small \bf Our Lower Bound} \tabularnewline
            \hline
            \textbf{Standard}
            
            Cumulative Regret\footnotemark\setcounter{auxFootnote}{\value{footnote}}\textsuperscript{,}\footnotemark\setcounter{auxFootnote2}{\value{footnote}} & $O^{*}\Big(\sqrt{T(\log T)^{2d} \log\frac{1}{\delta}} \Big)$ & $\Omega\Big(\sqrt{T(\log T)^{d/2}}\Big)$ & $\Omega^{*}\Big(\sqrt{T(\log T)^{d/2}\log\frac{1}{\delta}}\Big)$\tabularnewline
            \hline 
            \textbf{Corrupted Samples}
            
            Cumul.~Regret, $\delta = \Theta(1)$ & $O^{*}\Big(\overline{R}_T^{\rm std}+C\sqrt{T(\log T)^{d}}\Big)$ & $\Omega\Big(\underline{R}_T^{\rm std} + C\Big)$  & $\Omega\Big(\underline{R}_T^{\rm std} + C(\log T)^{d/2}\Big)$ \tabularnewline
            \hline 
            \textbf{Corrupted Final Point}
            
            Time to $\epsilon$-optimality\footnotemark[\value{auxFootnote2}] & $O^{*}\Big(\frac{1}{\epsilon^{2}}\big(\log\frac{1}{\epsilon}\big)^{2d}\log\frac{1}{\delta}\Big)$  & $\Omega\Big(\frac{1}{\epsilon^{2}}\big(\log\frac{1}{\epsilon}\big)^{\frac{d}{2}}\Big)$
            
            \textbf{(only for $\delta\le O(\xi^d)$)} & $\Omega\Big(\frac{1}{\epsilon^{2}}\big(\log\frac{1}{\epsilon}\big)^{\frac{d}{2}}\log\frac{1}{\delta}\Big)$\tabularnewline
            \hline 
            
        \end{tabular}
        \par\end{centering}
    
    \medskip
    {\centering \bf \underline{Mat\'ern-$\nu$ kernel} \par} 
    \smallskip
    \begin{centering}
    \begin{tabular}{|>{\centering}m{3.6cm}|>{\centering}m{4.8cm}|>{\centering}m{3.4cm}|>{\centering}m{3.6cm}|}
        
        \hline 
        & {\small \bf Upper Bound  } & {\small \bf Existing Lower Bound } & {\small \bf Our Lower Bound} \tabularnewline
        \hline
        \textbf{Standard}
        
        Cumulative Regret\footnotemark[\value{auxFootnote}] & $O^{*}\Big(T^{\frac{\nu+d}{2\nu+d}} \sqrt{\log\frac{1}{\delta}}\Big) $ & $\Omega\Big(T^{\frac{\nu+d}{2\nu+d}}\Big)$ & $\Omega\Big(T^{\frac{\nu+d}{2\nu+d}}\big(\log\frac{1}{\delta}\big)^{\frac{\nu}{2\nu+d}}\Big)$ \tabularnewline
        \hline 
        \textbf{Corrupted Samples}
        
        Cumul.~Regret, $\delta = \Theta(1)$ & $O^{*}\Big(\overline{R}_T^{\rm std}+CT^{\frac{\nu+d}{2\nu+d} } \Big)$ & $\Omega\Big(\underline{R}_T^{\rm std} + C\Big)$  & $\Omega\Big(\underline{R}_T^{\rm std} + C^{\frac{\nu}{d+\nu}}T^{\frac{d}{d+\nu}}\Big)$ \tabularnewline
        \hline 
        \textbf{Corrupted Final Point}
        
        Time to $\epsilon$-optimality & $O^{*}\Big(\big(\frac{1}{\epsilon}\big)^{\frac{2(2\nu+d)}{2\nu-d}} + \big(\frac{\log\frac{1}{\delta}}{\epsilon^{2}}\big)^{1+\frac{d}{2\nu}}\Big)$
        
        \textbf{(only for $d < 2\nu$)}  & $\Omega\Big(\frac{1}{\epsilon^{2}}\big(\frac{1}{\epsilon}\big)^{d/\nu}\Big)$
        
        \textbf{(only for $\delta\le O(\xi^d)$)} & $\Omega\Big(\frac{1}{\epsilon^{2}}\big(\frac{1}{\epsilon}\big)^{d/\nu}\log\frac{1}{\delta}\Big)$\tabularnewline
        \hline 
        
    \end{tabular}
    \par\end{centering}
    
    \protect\protect\caption{Summary of new and existing regret bounds. $T$ denotes the time horizon, $d$ denotes the dimension, $\xi$ denotes the corruption radius, and $\delta$ denotes the allowed error probability.  In the middle row, $\overline{R}_T^{\rm std}$ and $\underline{R}_T^{\rm std}$ denote upper and lower bounds on the standard cumulative regret.
        The existing upper and lower bounds are from \cite{Sri09,Cho17,Sca17a,Bog18,Bog20}, with the partial exception of the Mat\'ern kernel upper bounds, which are detailed at the end of Appendix \ref{sec:improved_matern}.  The notation $O^*(\cdot)$ and $\Omega^*(\cdot)$ hides dimension-independent $\log T$ factors, as well as $\log\log\frac{1}{\delta}$ factors. \label{tbl:standard}}
\end{table*}

\subsection{Standard Setting} \label{sec:std}

Let $f$ be a function on the compact domain $D = [0,1]^d$; by simple re-scaling, the results that we state readily extend to other rectangular domains.  The smoothness of $f$ is modeled by assuming that $\|f\|_{k} \le B$, where $\|\cdot\|_k$ is the RKHS norm associated with some kernel function $k(\xv,\xv')$ \cite{Ras06}.  The set of all functions satisfying $\|f\|_{k} \le B$ is denoted by $\Fc_k(B)$, and $\xv^*$ denotes an arbitrary maximizer of $f$.

At each round indexed by $t$, the algorithm selects some $\xv_t \in D$, and observes a noisy sample $y_t = f(\xv_t) + z_t$.  Here the noise term is distributed as $N(0,\sigma^2)$, with $\sigma^2 > 0$ and independence between times.

We measure the performance using the following two widespread notions of regret:
\begin{itemize}[leftmargin=5ex,itemsep=0ex,topsep=0.25ex]
    \item {\bf Simple regret:} After $T$ rounds, an additional point $\xv^{(T)}$ is returned, and the simple regret is given by $r(\xv^{(T)}) = f(\xv^*) - f(\xv^{(T)})$.  
    \item {\bf Cumulative regret:} After $T$ rounds, the cumulative regret incurred is $R_T = \sum_{t=1}^T r_t$, where $r_t = f(\xv^*) - f(\xv_t)$.  
\end{itemize}

As with the previous work on noisy lower bounds \cite{Sca17a}, we focus on the squared exponential (SE) and Mat\'ern kernels, defined as follows with length-scale $l > 0$ and smoothness $\nu > 0$ \cite{Ras06}: 
\begin{align}
    \kSE(\xv,\xv') &= \exp \bigg(- \dfrac{r_{\xv,\xv'}^2}{2l^2} \bigg) \label{eq:kSE} \\ 
    \kMat(\xv,\xv') &= \dfrac{2^{1-\nu}}{\Gamma(\nu)} \bigg(\dfrac{\sqrt{2\nu}\,r_{\xv,\xv'}}{l}\bigg)^{\nu}  J_{\nu}\bigg(\dfrac{\sqrt{2 \nu}\,r_{\xv,\xv'}}{l} \bigg), \label{eq:kMat}
    \end{align}
where $r_{\xv,\xv'} = \|\xv - \xv'\|$, and $J_{\nu}$ denotes the modified Bessel function.  

{\bf Existing lower bounds.} The results of \cite{Sca17a} are informally summarized as follows:
\begin{itemize}[leftmargin=5ex,itemsep=0ex,topsep=0.2ex]
    \item Attaining (average or constant-probability) simple regret $\epsilon$ requires the time horizon to satisfy $T = \Omega\big(\frac{1}{\epsilon^{2}}\big(\log\frac{1}{\epsilon}\big)^{d/2}\big)$  for the SE kernel, and $T = \Omega\big(\big(\frac{1}{\epsilon}\big)^{2+d/\nu}\big)$ for the Mat\'ern kernel.
    \item The (average or constant-probability) cumulative regret is lower bounded according to $R_T = \Omega\big(\sqrt{T(\log T)^{d/2}}\big)$ for the SE kernel, and $R_T = \Omega\big(T^{\frac{\nu+d}{2\nu+d}}\big)$ for the Mat\'ern kernel.
\end{itemize}
The SE kernel bounds have near-matching upper bounds \cite{Sri09}, and while standard results yield wider gaps for the Mat\'ern kernel, these have been tightened in recent works; see Appendix \ref{app:existing} for details.

In Sections \ref{sec:ana_std_simple}--\ref{sec:simplified_matern}, we will present novel analysis techniques that can both simplify the proofs and strengthen the dependence on the error probability compared to the lower bounds in \cite{Sca17a}.\footnote{We are not aware of any way to adapt the analysis of \cite{Sca17a} to obtain a high-probability lower bound that grows unbounded as the target error probability approaches zero.}

\subsection{Robust Setting -- Corrupted Samples} \label{sec:setup_corr}

In the robust setting studied in \cite{Bog20}, the optimization goal is similar, but each sampled point is further subject to adversarial noise; for $t=1,\dots,T$:
\begin{itemize}[leftmargin=5ex,itemsep=0ex,topsep=0.25ex]
    \item Based on the previous samples $\lbrace (\xv_i, \ytil_i) \rbrace_{i=1}^{t-1}$, the player selects a distribution $\Phi_t(\cdot)$ over $D$.
  \item Given knowledge of the true function $f$, the previous samples $\lbrace (\xv_i, y_i) \rbrace_{i=1}^{t-1}$, and the player's distribution $\Phi_t(\cdot)$, an adversary selects a function $c_t(\cdot):D \rightarrow [-B_0, B_0]$, where $B_0 > 0$ is constant.
  \item The player draws $\xv_t \in D$ from the distribution $\Phi_t$, and observes the corrupted sample
  \begin{equation}
    \label{eq:corrupted_observation}
    \ytil_t = y_t + c_t(\xv_t),
  \end{equation}
  where $y_t$ is the noisy non-corrupted observation $y_t = f(\xv_t) + z_t$ as in Section \ref{sec:std}.
\end{itemize}
Note that in the special case that $\Phi_t(\cdot)$ is deterministic, the adversary knowing $\Phi_t$ also implies knowledge of $\xv_t$.


For this problem to be meaningful, the adversary must be constrained.  Following \cite{Bog20}, we assume the following constraint for some corruption level $C$:
\begin{equation}
\label{eq:total_corruption}
  \sum_{t=1}^T \max_{\xv \in D} |c_{t}(\xv)| \leq C.
\end{equation}
When $C = 0$, we reduce to the setup of Section \ref{sec:std}.

While both the simple regret and cumulative regret could be considered here, we focus entirely on the latter, as it has been the focus of the related existing works \cite{Bog20,Bog20a,Lyk18,Gup19,Li19b}.  See also \citep[App.~C]{Bog20} for discussion on the use of simple regret in this setting.

{\bf Existing lower bound.} The only lower bound stated in \cite{Bog20} states that $R_T = \Omega(C)$ for any algorithm, whereas the upper bound therein essentially amounts to {\em multiplying} (rather than adding) the uncorrupted regret bound by $C$.  Thus, significant gaps remain in terms of the {\em joint} dependence on $C$ and $T$, which our lower bound in Section \ref{sec:corr_setting} will partially address.

\subsection{Robust Setting -- Corrupted Final Point} \label{sec:setup_adv}

\footnotetext[\value{auxFootnote}]{Analogous results are also given for the standard simple regret (time to $\epsilon$-optimality).}
\footnotetext[\value{auxFootnote2}]{Here we have presented simplified and slightly loosened forms; the refined variants are stated at the end of Appendix \ref{sec:improved_matern}.}
\setcounter{footnote}{2}

Here we detail a different robust setting, previously considered in \cite{Bog18}, in which the samples themselves are only subject to random (non-adversarial) noise, but the {\em final point} returned may be adversarially perturbed.  For a real-valued function $\dist(\xv,\xv')$ and constant $\xi$, we define the set-valued function
\begin{equation}
    \Delta_{\xi}(\xv) = \big\{ \xv' - \xv \,:\, \xv' \in D \;\; \text{and}\;\; \dist(\xv,\xv') \le \xi \big\}
\end{equation}
representing the set of perturbations of $\xv$ such that the newly obtained point $\xv'$ is within a ``distance'' $\xi$ of $\xv$.  

We seek to attain a function value as high as possible following the worst-case perturbation within $\Delta_{\xi}(\cdot)$; in particular, the global robust optimizer is given by
\begin{equation}
    \xv^*_{\xi} \in \argmax_{\xv \in D} \min_{\bdelta \in \Delta_{\xi}(\xv)} f(\xv + \bdelta). \label{eq:eps_stable_input}
\end{equation}
Then, if the algorithm returns  $\xv^{(T)}$, the performance is measured by the {\em $\xi$-regret}:
\begin{equation} 
  \label{eq:eps_regret}
    r_{\xi}(\xv) = \min_{\bdelta \in \Delta_{\xi}(\xv^*_{\xi})} f(\xv^*_{\xi} + \bdelta) - \min_{\bdelta \in \Delta_{\xi}(\xv)} f(\xv + \bdelta).
\end{equation}
We focus our attention on the primary case of interest  in which $\dist(\xv,\xv') = \|\xv - \xv'\|_2$, meaning that achieving low $\xi$-regret amounts to favoring {\em broad peaks} instead of narrow ones, particularly for higher $\xi$.  

While robust cumulative regret notions are possible \cite{Kir20}, we focus on the (simple) $\xi$-regret, as it was the focus of \cite{Bog18} and extensive related works \cite{Ses20,Ngu20,Ber10}.

{\bf Existing lower bound.} A lower bound is proved in \cite{Bog18} for the case of constant $\xi > 0$, with the same scaling as the standard setting.  However, \cite{Bog18} only proves this hardness result for succeeding with probability very close to one; our lower bound in Section \ref{sec:adv_robust_results} overcomes this limitation.

\section{Main Results}

In this section, we formally state the new lower bounds that are summarized in Table \ref{tbl:standard}.

\subsection{Standard Setting} \label{sec:res_std}

Our first contribution is to provide a new approach to establishing lower bounds in the standard setting (Section \ref{sec:std}), with several advantages compared to \cite{Sca17a} discussed in Section \ref{sec:cmp}.

In the standard multi-armed bandit problem with a finite number of independent arms, \citep[Lemma 1]{Kau16} gives a versatile tool for deriving regret bounds based on the data processing inequality for KL divergence (e.g., see \citep[Sec.~6.2]{Pol14}).  The idea is that if two bandit instances must produce different outcomes (e.g., a different final point $\xv^{(T)}$ must be returned) in order to succeed, but their sample distributions are close in KL divergence, then the time horizon must be large.

While \citep[Lemma 1]{Kau16} is only stated for a finite number of arms, the proof technique therein readily yields the variant in Lemma \ref{lem:relating} below for a continuous input space, with the KL divergence quantities defined by maximizing within each of a finite number of regions partitioning the space.  See also \cite{Azi18} for an extension of \citep[Lemma 1]{Kau16} to a different infinite-arm problem.

In the following, we let $\PP_f[\cdot]$ denote probabilities (with respect to the random noise) when the underlying function is $f$, and we let $P_f(y|\xv)$ be the conditional distribution $N(f(\xv),\sigma^2)$ according to the Gaussian noise model.

\begin{lem} \label{lem:relating}
    {\em (Relating Two Instances -- Adapted from \citep[Lemma 1]{Kau16})}
    Fix $f, f' \in \Fc_k(B)$, let $\{\Rc_j\}_{j=1}^M$ be a partition of the input space into $M$ disjoint regions, and let $\Ac$ be any event depending on the history up to some almost-surely finite stopping time $\tau$.\footnote{Following \cite{Kau16}, we state this result for general algorithms that are allowed to choose when to stop.  Our focus in this paper is on the fixed-length setting in which the time horizon is pre-specified, and this setting is recovered by simply setting $\tau = T$ deterministically.}  Then, for $\delta \in \big(0,\frac{1}{3}\big)$, if $\PP_{f}[\Ac] \ge 1-\delta$ and $\PP_{f'}[\Ac] \le \delta$, we have
    \begin{equation}
        \sum_{j=1}^M \EE_{f}[N_j(\tau)] \Dbar^j_{f,f'} \ge \log\frac{1}{2.4 \delta},
    \end{equation}
    where $N_j(\tau)$ is the number of selected points in the $j$-th region up to time $\tau$, and
    \begin{equation}
        \Dbar^j_{f,f'} = \max_{\xv \in \Rc_j} D\big( P_{f}(\cdot | \xv) \,\|\, P_{f'}(\cdot | \xv) \big) \label{eq:Dbar_ff}
    \end{equation}
    is the maximum KL divergence between samples (i.e., noisy function values) from $f$ and $f'$ in the $j$-th region.
\end{lem}

In Section \ref{sec:proofs}, we will use Lemma \ref{lem:relating} to prove the following lower bounds on the simple regret and cumulative regret, which are similar to those of \cite{Sca17a} but enjoy an improved $\log\frac{1}{\delta}$ dependence on the target error probability $\delta$.  Despite this improvement, we highlight that the key contribution in this part of the paper is the novel lower bounding techniques for GP bandits via Lemma \ref{lem:relating}, rather than the results themselves.  See Section \ref{sec:cmp} for a comparison to the approach of \cite{Sca17a}.

\begin{thm} \label{thm:simple_lb_new}
    \emph{(Simple Regret Lower Bound -- Standard Setting)}
    Fix $\delta \in \big(0,\frac{1}{3}\big)$, $\epsilon \in \big(0,\frac{1}{2}\big)$, $B > 0$, and $T \in \ZZ$.   Suppose there exists an algorithm that, for any $f \in \Fc_k(B)$, achieves average simple regret $r(\xv^{(T)}) \le \epsilon$ with probability at least $1 - \delta$.  Then, if $\frac{\epsilon}{B}$ is sufficiently small, we have the following:
    \begin{enumerate}
        \item For \emph{$k = \kSE$}, it is necessary that
        \begin{equation}
            T = \Omega\bigg( \frac{\sigma^2}{\epsilon^2} \Big(\log\frac{B}{\epsilon}\Big)^{d/2} \log\frac{1}{\delta} \bigg). \label{eq:inst_se_new}
       \end{equation}
        \item For \emph{$k = \kMat$}, it is necessary that
        \begin{equation}
            T = \Omega\bigg( \frac{\sigma^2}{\epsilon^2} \Big(\frac{B}{\epsilon}\Big)^{d/\nu} \log\frac{1}{\delta} \bigg). \label{eq:inst_Mat_new}
       \end{equation}
    \end{enumerate}
    Here, the implied constants may depend on $(d,l,\nu)$.
\end{thm}

\begin{thm} \label{thm:cumul_lb_new}
    \emph{(Cumulative Regret Lower Bound -- Standard Setting)}
    Given $T \in \ZZ$, $\delta \in \big(0,\frac{1}{3}\big)$, and $B > 0$, for any algorithm, we must have the following:
    \begin{enumerate}
        \item For $k = \kSE$, there exists $f \in \Fc_k(B)$ such that the following holds with probability at least $\delta$:\footnote{This ``failure'' event occurring with probability $\delta$ implies that the algorithm is unable to attain a $(1-\delta)$-probability of ``success''.}
        \begin{equation}
            R_T = \Omega\Bigg( \sqrt{T\sigma^2 \Big(\log \frac{B^2 T}{\sigma^2 \log\frac{1}{\delta}} \Big)^{d/2} \log\frac{1}{\delta} } \Bigg) \label{eq:cumul_SE_new}
       \end{equation}
       provided that\footnote{As discussed in \cite{Sca17a}, scaling assumptions of this kind are very mild, and are needed to avoid the right-hand side of \eqref{eq:cumul_SE_new} contradicting a trivial $O(BT)$ upper bound.} $\frac{\sigma^2 \log\frac{1}{\delta}}{B^2} = O(T)$ with a sufficiently small implied constant.
        \item For $k = \kMat$, there exists $f \in \Fc_k(B)$ such that the following holds with probability at least $\delta$:
        \begin{equation}
            R_T = \Omega\bigg( B^{\frac{d}{2\nu + d}} T^{\frac{\nu + d}{2\nu + d}} \Big(\sigma^2\log\frac{1}{\delta}\Big)^{\frac{\nu}{2\nu + d}}  \bigg) \label{eq:cumul_Mat_new}
       \end{equation}
       provided that $\frac{\sigma^2 \log\frac{1}{\delta}}{B^2} = O\big(T^{\frac{1}{2+d/\nu}} \big)$ with a sufficiently small implied constant.
    \end{enumerate}
    Here, the implied constants may depend on $(d,l,\nu)$.
\end{thm}

In Section \ref{sec:simplified_matern}, we show that for the Mat\'ern kernel, the analysis can be simplified even further by using a function class proposed in \cite{Bul11} (which studied the noiseless setting) with bounded support.

\subsection{Robust Setting -- Corrupted Samples} \label{sec:res_corr}

In the general setup studied in \cite{Bog20} and presented in Section \ref{sec:setup_corr}, the player may randomize the choice of action, and the adversary can know the distribution but not the specific action.  However, if the player's actions are deterministic (given the history), then knowing the distribution is equivalent to knowing the specific action.  In this section, we provide a lower bound for such scenarios.  While a lower bound that only holds for deterministic algorithms may seem limited, it is worth noting that the smallest regret upper bound in \cite{Bog20} (see Theorem \ref{thm:known_C_ub} in Appendix \ref{app:existing}) is established using such an algorithm.  More generally, it is important to know to what extent randomization is needed for robustness, so bounds for both deterministic and randomized algorithms are of significant interest.

\smallskip
\begin{thm} \label{thm:corr_lower}
    \emph{(Lower Bound -- Corrupted Samples)}
    In the setting of corrupted samples with a corruption level satisfying $\Theta(1) \le C \le T^{1 - \Omega(1)}$, even in the noiseless setting ($\sigma^2 = 0$), any deterministic algorithm (including those having knowledge of $C$) yields the following with probability one for some $f \in \Fc_k(B)$:
    \begin{itemize}[leftmargin=5ex,itemsep=0ex,topsep=0.2ex]
        \item Under the SE kernel, $R_T = \Omega\big( C (\log T)^{d/2} \big)$;
        \item Under the Mat\'ern-$\nu$ kernel, $R_T = \Omega\big( C^{\frac{\nu}{d+\nu}} T^{\frac{d}{d+\nu}}  \big)$.
    \end{itemize}
\end{thm}
\smallskip


We provide a proof outline in Section \ref{sec:corr_setting}, and the full details in Appendix \ref{app:pf_corr}.  We note that the assumption $\Theta(1) \le C \le T^{1 - \Omega(1)}$ primarily rules out the case $C = \Theta(T)$ in which the adversary can corrupt every point by a constant amount. This assumption also ensures that the bound $R_T = \Omega\big( C^{\frac{\nu}{d+\nu}} T^{\frac{d}{d+\nu}}  \big)$ is stronger than the bound $R_T = \Omega(C)$ from \cite{Bog20}. Note also that any lower bound for the standard setting applies here, since the adversary can choose not to corrupt.

Theorem \ref{thm:corr_lower} addresses a question posed in \cite{Bog20} on the joint dependence of the cumulative regret on $C$ and $T$.  The upper bounds established therein (one of which we replicate in Theorem \ref{thm:known_C_ub} in Appendix \ref{app:existing}) are of the form $O( C \Rbar^{(0)} )$, where $\Rbar^{(0)}$ is a standard (non-corrupted) regret bound, whereas analogous results from the multi-armed bandit literature \cite{Gup19} suggest that $\tilde{O}( \Rbar^{(0)} + C)$ may be possible, where the $\tilde{O}(\cdot)$ notation hides dimension-independent logarithmic factors.

Theorem \ref{thm:corr_lower} shows that, at least for deterministic algorithms, such a level of improvement is impossible in the RKHS setting.  On the other hand, further gaps remain between the lower bounds in Theorem \ref{thm:corr_lower} and the $O( C \Rbar^{(0)} )$ upper bounds of \cite{Bog20} (e.g., for the SE kernel, the latter introduces an $\tilde{O}(C \sqrt{T (\log  T)^{2d}})$ term, whereas Theorem \ref{thm:corr_lower} gives an $\Omega\big( C (\log T)^{d/2} \big)$ lower bound).  Recent results for the linear bandit setting \cite{Bog20a} suggest that the looseness here may be in the upper bound; this is left for future work.

\subsection{Robust Setting -- Corrupted Final Point} \label{sec:res_adv}

Here we provide improved variant of the lower bound in \cite{Bog18} (replicated in Theorem \ref{thm:lower_robust} in Appendix \ref{app:existing}) for the adversarially robust setting with a corrupted final point, described in Section \ref{sec:setup_adv}.

\smallskip
\begin{thm} \label{thm:lower_robust_new}
    \emph{(Improved Lower Bound -- Corrupted Final Point)}
    Fix $\xi \in \big(0,\frac{1}{2}\big)$, $\epsilon \in \big(0,\frac{1}{2}\big)$, $B > 0$, and $T \in \ZZ$, and set $\dist(\xv,\xv') = \|\xv - \xv'\|_2$.   Suppose that there exists an algorithm that, for any $f \in \Fc_k(B)$, reports $\xv^{(T)}$ achieving $\xi$-regret $r_{\xi}(\xv^{(T)}) \le \epsilon$ with probability at least $1 - \delta$.  Then, provided that $\frac{\epsilon}{B}$ is sufficiently small, we have the following:
    \begin{enumerate}[leftmargin=5ex,itemsep=0ex,topsep=0.2ex]
        \item For \emph{$k = \kSE$}, it is necessary that
        $T = \Omega\big( \frac{\sigma^2}{\epsilon^2} \big(\log\frac{B}{\epsilon}\big)^{d/2} \log\frac{1}{\delta} \big)$. 
        \item For \emph{$k = \kMat$}, it is necessary that
        $T = \Omega\big( \frac{\sigma^2}{\epsilon^2} \big(\frac{B}{\epsilon}\big)^{d/\nu} \log\frac{1}{\delta} \big).$ 
    \end{enumerate}
    Here, the implied constants may depend on $(\xi,d,l,\nu)$.
\end{thm}
\smallskip

Compared to the existing lower bound in \cite{Bog18} (Theorem \ref{thm:lower_robust} in Appendix \ref{app:existing}), we have removed the restrictive requirement that $\delta$ is sufficiently small (see Appendix \ref{sec:adv_robust_results} for further discussion), giving the same scaling laws even when the algorithm is only required to succeed with a small probability such as $0.01$ (i.e, $\delta = 0.99$).  In addition, for small $\delta$, we attain a $\log\frac{1}{\delta}$ factor improvement similar to Theorem \ref{thm:simple_lb_new}.

\section{Mathematical Analysis and Proofs} \label{sec:proofs}

\subsection{Preliminaries} \label{sec:prelim}

\begin{figure}
    \begin{centering}
        \includegraphics[width=0.4\textwidth]{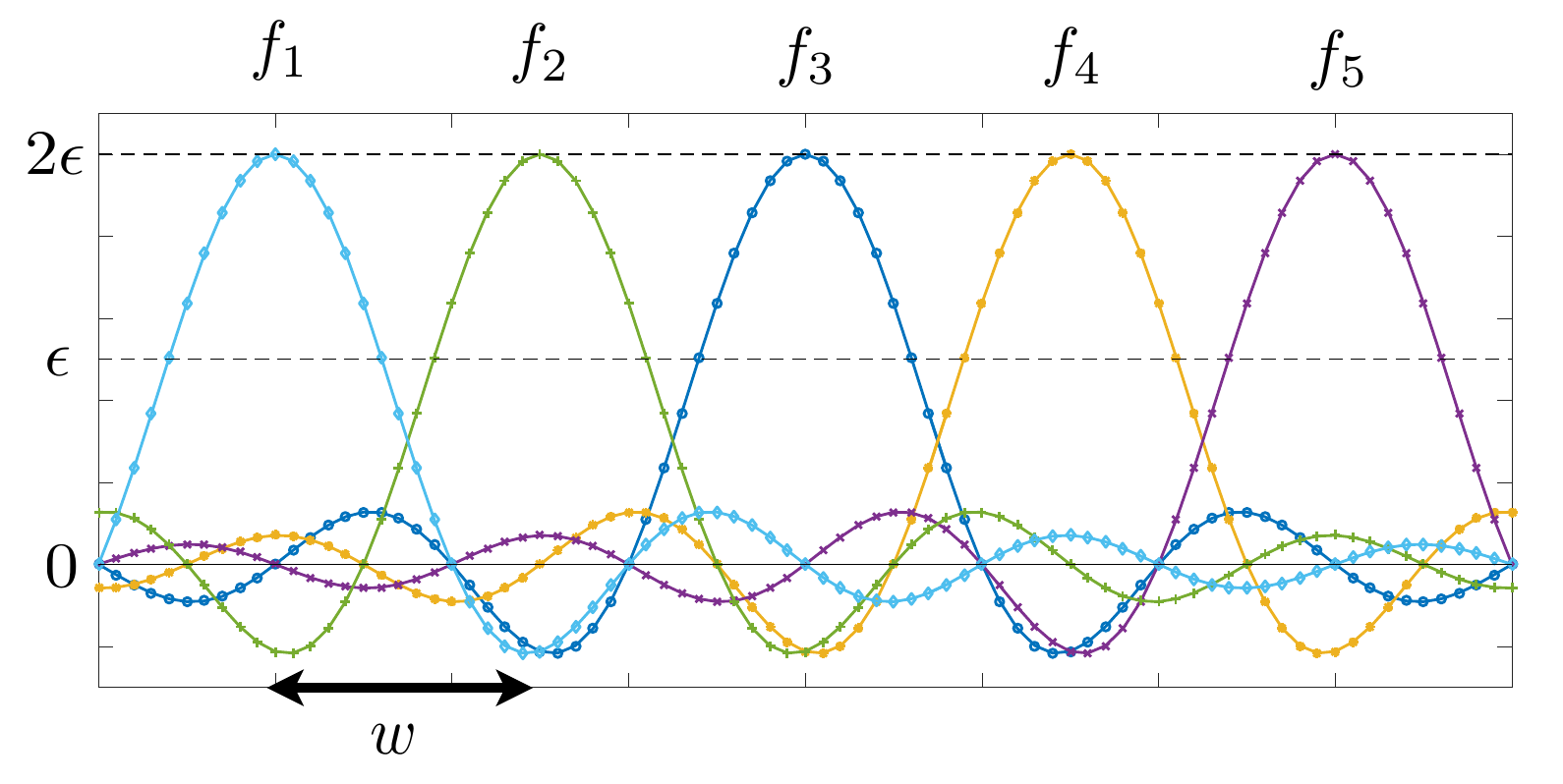}
        \par
    \end{centering}
    
    \caption{Illustration of functions $f_1,\dotsc,f_5$ such that any given point is $\epsilon$-optimal for at most one function. \label{fig:func_class_standard}} 
\end{figure}

In this section, we introduce some preliminary auxiliary results from \cite{Sca17a} that will be used throughout our analysis.  While we utilize the function class and auxiliary results from this existing work, we apply them in a significantly different manner in order to broaden the limited techniques known for GP bandit lower bounds, and to reap the advantages outlined above and in Section \ref{sec:cmp}.


We proceed as follows \cite{Sca17a}:
\begin{itemize}[leftmargin=5ex,itemsep=0ex,topsep=0.2ex]
    \item We lower bound the worst-case regret within $\Fc_k(B)$ by the regret averaged over a finite collection $\{f_1,\dotsc,f_M\} \subset \Fc_k(B)$ of size $M$.
    \item Except where stated otherwise, we choose each $f_m(\xv)$ to be a shifted version of a common function $g(\xv)$ on $\RR^d$.  Specifically, each $f_m(\xv)$ is obtained by shifting $g(\xv)$ by a different amount, and then cropping to $D = [0,1]^d$.  For our purposes, we require $g(\xv)$ to satisfy the following properties:
    \begin{enumerate}
        \item The RKHS norm in $\RR^d$ satisfies $\|g\|_k \le B$;
        \item We have (i) $g(\xv) \in [-2\epsilon,2\epsilon]$ with maximum value $g(0) = 2\epsilon$, and (ii) there is a ``width'' $w$ such that $g(\xv) < \epsilon$ for all $\|\xv\|_{\infty} \ge \frac{w}{2}$;
        \item There are absolute constants $h_0 > 0$ and $\zeta > 0$ such that $g(\xv) = \frac{2\epsilon}{h_0} h\big(\frac{\xv\zeta}{w}\big)$ for some function $h(\zv)$ that decays faster than any finite power of $\|\zv\|_2^{-1}$ as $\|\zv\|_2 \to \infty$.
    \end{enumerate}
    Letting $g(\xv)$ be such a function, we construct the $M$ functions by shifting $g(\xv)$ so that each $f_m(\xv)$ is centered on a unique point in a uniform grid, with points separated by $w$ in each dimension.  Since $D = [0,1]^d$, one can construct
    \begin{equation}
        M = \Big\lfloor \Big( \frac{1}{w} \Big)^d \Big\rfloor \label{eq:Mw}
    \end{equation}
    such functions; we will always consider $w \ll 1$, so that the case $M = 0$ is avoided.  See Figure \ref{fig:func_class_standard} for an illustration of the function class.
    \item It is shown in \cite{Sca17a} that the above properties can be achieved with
    \begin{equation}
        M = \Bigg\lfloor \Bigg( \frac{ \sqrt{\log\frac{B (2\pi l^2)^{d/4} h(0)}{2\epsilon}} }{\zeta \pi l} \Bigg)^d \Bigg\rfloor \label{eq:M_se}
    \end{equation}
    in the case of the SE kernel, and with 
    \begin{equation}
        M = \Big\lfloor \Big( \frac{B c_3}{\epsilon} \Big)^{d/\nu} \Big\rfloor \label{eq:M_matern}
    \end{equation}
    in the case of the Mat\'ern kernel, where $c_3 :=  \big( \frac{1}{\zeta} \big)^{\nu} \cdot \big( \frac{ c_2^{-1/2} }{ 2 (8\pi^2)^{(\nu + d/2)/2} } \big)$, and where $c_2 > 0$ is an absolute constant.  Note that these values of $M$ amount to choosing $w$ in \eqref{eq:Mw}, and we will always consider $\frac{\epsilon}{B}$ to be sufficiently small, thus ensuring that $M \gg 1$ and  $w \ll 1$ as stated above.
\end{itemize}
In addition, we introduce the following notation:
\begin{itemize}[leftmargin=5ex,itemsep=0ex,topsep=0.2ex]
    \item The probability density function of the output sequence $\yv = (y_1,\dotsc,y_T)$ when $f=f_m$ is denoted by $P_m(\yv)$ (and implicitly depends on the arbitrary underlying bandit algorithm).  We also define $f_0(\xv) = 0$ to be the zero function, and define $P_0(\yv)$ analogously for the case that the optimization algorithm is run on $f_0$.  Expectations and probabilities (with respect to the noisy observations) are similarly written as $\EE_m$, $\PP_m$, $\EE_0$, and $\PP_0$ when the underlying function is $f_m$ or $f_0$.  On the other hand, in the absence of a subscript, $\EE$ and $\PP$ are taken with respect to the noisy observations {\em and} the random function $f$ drawn uniformly from $\{f_1,\dotsc,f_M\}$.  In addition, $\PP_f$ and $\EE_f$ will sometimes be used for generic $f$.
    \item Let $\{\Rc_m\}_{m=1}^M$ be a partition of the domain into $M$ regions according to the above-mentioned uniform grid, with $f_m$ taking its minimum value of $-2\epsilon$ in the center of $\Rc_m$.  Moreover, let $j_t$ be the index at time $t$ such that $\xv_t$ falls into $\Rc_{j_t}$; this can be thought of as a quantization of $\xv_t$.
    \item Define the maximum absolute function value within a given region $\Rc_j$ as
    \begin{equation}
    \vbar_m^j := \max_{\xv \in \Rc_j} |f_m(\xv)|, \label{eq:vbar}
    \end{equation}
    and the maximum KL divergence to $P_0$ within $\Rc_j$ as
    \begin{equation}
    \Dbar_m^j := \max_{\xv \in \Rc_j} D( P_0(\cdot|\xv) \| P_m(\cdot|\xv) ), \label{eq:Dbar}
    \end{equation}
    where $P_m(y|\xv)$ is the distribution of an observation $y$ for a given selected point $\xv$ under the function $f_m$, and similarly for $P_0(y|\xv)$.
    \item Let $N_j \in \{0,\dotsc,T\}$ be a random variable representing the number of points from $\Rc_j$ that are selected throughout the $T$ rounds.
\end{itemize}

Finally, the following auxiliary lemmas will be useful.

\begin{lem} \label{lem:Gaussian_div}
    {\em \citep[Eq.~(36)]{Sca17a}}
    For $P_1$ and $P_2$ being Gaussian with means $(\mu_1,\mu_2)$ and a common variance $\sigma^2$, we have $D(P_1 \| P_2) = \frac{ (\mu_1 - \mu_2)^2 }{ 2\sigma^2 }$.
\end{lem}

\begin{lem} \label{lem:vbar_sums}
    {\em \citep[Lemma 7]{Sca17a}}
	The functions $\{f_m\}$ corresponding to \eqref{eq:M_se}--\eqref{eq:M_matern} are such that the quantities $\vbar_m^j$ in \eqref{eq:vbar} satisfy (i) $\sum_{j=1}^M \vbar_m^j = O(\epsilon)$ for all $m$; (ii) $\sum_{m=1}^M \vbar_m^j = O(\epsilon)$ for all $j$; and (iii) $\sum_{m=1}^M (\vbar_m^j)^2 = O(\epsilon^2)$ for all $j$.
\end{lem}


\subsection{Standard Setting -- Simple Regret} \label{sec:ana_std_simple}

For fixed $\epsilon > 0$, we consider the function class $\{f_1,\dotsc,f_M\}$ described in Section \ref{sec:prelim}, with $B$ replaced by $\frac{B}{3}$.  This change should be understood as applying to all previous equations and auxiliary results that we use, e.g., replacing $B$ by $\frac{B}{3}$ in \eqref{eq:M_matern}, but this only affects constant factors, which we do not attempt to optimize anyway.  Before continuing, we recall the important property that any $\xv \in [0,1]^d$ can be $\epsilon$-optimal for at most one function.

For fixed $m$ and $m'$, we will apply Lemma \ref{lem:relating} with $f(\xv) = f_m(\xv)$ and $f'(\xv) = f_m(\xv) + 2f_{m'}(\xv)$.  Intuitively, this choice is made so that $f$ and $f'$ have different maximizers, but remain near-identical except in a small region around the peak of $f'$.  It will be useful to characterize the quantity $\Dbar^j_{f,f'}$ in \eqref{eq:Dbar_ff}; by Lemma \ref{lem:Gaussian_div},
\begin{equation}
    \Dbar^j_{f,f'} = \max_{\xv \in \Rc_j} \frac{|2f_{m'}(\xv)|^2}{2\sigma^2} = \frac{2(\vbar_{m'}^j)^2}{\sigma^2}, \label{eq:eval_Dbar}
\end{equation}
using the definition of $\vbar_m^j$ in \eqref{eq:vbar}.

In the following, let $\Ac$ be the event that the returned point $\xv^{(T)}$ lies in the region $\Rc_m$ (defined just above \eqref{eq:vbar}). Suppose that an algorithm attains simple regret at most $\epsilon$ for both $f$ and $f'$ (note that $\|f'\|_k \le B$ by the triangle inequality and $\max_j \|f_j\|_k \le \frac{B}{3}$), each with probability at least $1-\delta$.   We claim that this implies $\PP_f[\Ac] \ge 1-\delta$ and $\PP_{f'}[\Ac] \le \delta$.  Indeed, by construction in Section \ref{sec:prelim}, only points in $\Rc_m$ can be $\epsilon$-optimal under $f = f_m$, and only points in $\Rc_{m'}$ can be $\epsilon$-optimal under $f' = f_m + 2f_{m'}$.  Hence, Lemma \ref{lem:relating} and \eqref{eq:eval_Dbar} give
\begin{equation}
    \frac{2}{\sigma^2} \sum_{j=1}^M \EE_{m}[N_j] \cdot (\vbar^j_{m'})^2 \ge \log\frac{1}{2.4 \delta}, \label{eq:std_simple_pf2}
\end{equation}
and summing over all $m' \ne m$ gives
\begin{equation}
    \frac{2}{\sigma^2} \sum_{m' \ne m} \sum_{j=1}^M \EE_{m}[N_j] \cdot (\vbar^j_{m'})^2 \ge (M-1) \log\frac{1}{2.4 \delta}. \label{eq:std_simple_pf3}
\end{equation}
Swapping the summations, using Lemma \ref{lem:vbar_sums} to upper bound $\sum_{m' \ne m} (\vbar^j_{m'})^2 \le O(\epsilon^2)$, and applying $\sum_{j=1}^M \EE_{m}[N_j(\tau)] = T$, we obtain 
\begin{equation}
    \frac{2c_0 \epsilon^2 T}{\sigma^2}  \ge (M-1) \log\frac{1}{2.4 \delta}
\end{equation}
for some constant $c_0$, or equivalently,
\begin{equation}
    T \ge \frac{(M-1)\sigma^2}{2c_0 \epsilon^2} \log\frac{1}{2.4 \delta}.  \label{eq:std_simple_pf5}
\end{equation}
Theorem \ref{thm:simple_lb_new} now follows using \eqref{eq:M_se} (with $\frac{B}{3}$ in place of $B$) for the SE kernel, or \eqref{eq:M_matern} for the Mat\'ern-$\nu$ kernel.

\begin{rem} \label{rem:variableT}
    The preceding analysis can easily be adapted to show that when $T$ is allowed to have variable length (i.e., the algorithm is allowed to choose when to stop), $\EE[T]$ is lower bounded by the right-hand side of \eqref{eq:std_simple_pf5}.  See \cite{Gab12} for a discussion on analogous variations in the context of multi-armed bandits.
\end{rem}

\subsection{Standard Setting -- Cumulative Regret} \label{sec:ana_std_cumul}

We fix some $\epsilon > 0$ to be specified later, consider the function class $\{f_1,\dotsc,f_M\}$ from Section \ref{sec:prelim}, and show that it is not possible to attain $R_T \le \frac{T\epsilon}{2}$ with probability at least $1-\delta$ for all functions with $\|f\|_k \le B$.

Assuming by contradiction that the preceding goal is possible, this class of functions includes the choices of $f$ and $f'$ at the start of Section \ref{sec:ana_std_simple}.  However, if we let $\Ac$ be the event that at least $\frac{T}{2}$ of the sampled points lie in $\Rc_m$, it follows that $\PP_f[\Ac] \ge 1-\delta$ and $\PP_{f'}[\Ac] \le \delta$, since (i) each sample outside $\Rc_m$ incurs regret at least $\epsilon$ under $f$; and (ii) each sample within $\Rc_m$ incurs regret at least $\epsilon$ under $f'$.  

Hence, despite being derived with a different choice of $\Ac$, \eqref{eq:std_simple_pf2} still holds in this case, and \eqref{eq:std_simple_pf5} follows.  This was derived under the assumption that $R_T \le \frac{T\epsilon}{2}$ with probability at least $1-\delta$ for all functions with $\|f\|_k \le B$; the contrapositive statement is that when 
\begin{equation}
    T < \frac{(M-1)\sigma^2}{2c_0 \epsilon^2} \log\frac{1}{2.4 \delta}, \label{eq:std_cumu_pf1}
\end{equation}
it must be the case that some function yields $R_T > \frac{T\epsilon}{2}$ with probability at least $\delta$.

The remainder of the proof of Theorem \ref{thm:cumul_lb_new} follows that of \citep[Sec.~5.3]{Sca17a}, but with $\sigma^2 \log\frac{1}{\delta}$ in place of $\sigma^2$, and the final regret expressions adjusted accordingly (e.g., compare Theorem \ref{thm:cumul_lb_new} with Theorem \ref{thm:cumul_lb} in Appendix \ref{app:existing}).  Due to this similarity, we only outline the details:
\begin{itemize}[leftmargin=5ex,itemsep=0ex,topsep=0.2ex]
    \item[(i)] Consider \eqref{eq:std_cumu_pf1} nearly holding with equality (e.g., $T = \frac{M\sigma^2}{4c_0 \epsilon^2} \log\frac{1}{2.4 \delta}$ suffices).
    \item[(ii)] Substitute $M$ from \eqref{eq:M_se} or \eqref{eq:M_matern} (with $\frac{B}{3}$ in place of $B$) into this choice from $T$, and solve to get an asymptotic expression for $\epsilon$ in terms of $T$.
    \item[(iii)] Substitute this expression for $\epsilon$ into $R_T > \frac{T\epsilon}{2}$ to obtain the final regret bound. 
\end{itemize}
See also Appendix \ref{app:pf_corr} for similar steps given in more detail, albeit in the robust setting.

\subsection{Comparison of Proof Techniques} \label{sec:cmp}

The above analysis borrows ingredients from \cite{Sca17a} and establishes similar final results; the key difference is in the use of Lemma \ref{lem:relating} in place of an additive change-of-measure result (see Lemma \ref{lem:auer} in Appendix \ref{app:further}).  We highlight the following advantages of our approach:
\begin{itemize}[leftmargin=5ex,itemsep=0ex,topsep=0.2ex]
    \item While both approaches can be used to lower bound the average or constant-probability regret,\footnote{Under the new proof given here, this is achieved by setting $\delta = \frac{1}{2}$, or any other fixed constant in $(0,1)$.} the above analysis gives the more precise $\log\frac{1}{\delta}$ dependence when the algorithm is required to succeed with probability at least $1-\delta$.
    \item As highlighted in Remark \ref{rem:variableT}, the above simple regret analysis extends immediately to provide a lower bound on $\EE[T]$ in the varying-$T$ setting, whereas attaining this via the approach of \cite{Sca17a} appears to be less straightforward.
    \item Although we do not explore it in this paper, we expect our approach to be more amenable to deriving {\em instance-dependent} regret bounds, rather than worst-case regret bounds over the function class.  The idea, as in the multi-armed bandit setting \cite{Kau16}, is that if we can ``perturb'' one function/instance to another so that the set of near-optimal points changes significantly, we can use Lemma \ref{lem:relating} to infer bounds on the required number of time steps in the original instance.
    \item As evidence of the versatility of our approach, we will use it in Section \ref{sec:adv_robust_setting} to derive an improved result over that of \cite{Bog18} in the robust setting with a corrupted final point.
\end{itemize}

\subsection{Simplified Analysis -- Mat\'ern Kernel} \label{sec:simplified_matern}

In the function class from \cite{Sca17a} used above, each function is a bump function (with bounded support) in the frequency domain, meaning that it is non-zero almost everywhere in the spatial domain.  In contrast, the earlier work of Bull \cite{Bul11} for the noiseless setting directly adopts a bump function in the spatial domain, permitting a simple analysis for the Mat\'ern kernel.  

It was noted in \cite{Sca17a} that such a choice is infeasible for the SE kernel, since its RKHS norm is infinite.  Nevertheless, in this section, we show that the (spatial) bump function is indeed much simpler to work with under the Mat\'ern kernel, not only in the noiseless setting of \cite{Bul11}, but also in the presence of noise.

The following result is stated in \citep[Sec.~A.2]{Bul11}, and follows using Lemma 5 therein.

\begin{lem} \label{lem:simpler_function}
    {\em (Bounded-Support Function Construction \cite{Bul11})}
    Let $h(\xv) = \exp\big( \frac{-1}{1-\|\xv\|^2} \big) \openone\{ \|\xv\|_2 < 1 \}$ be the $d$-dimensional bump function, and define $g(\xv) = \frac{2\epsilon}{h(\bzero)} h\big(\frac{\xv}{w}\big)$ for some $w > 0$ and $\epsilon > 0$. Then, $g$ satisfies the following properties:
    \begin{itemize}[leftmargin=5ex,itemsep=0ex,topsep=0.2ex]
        \item $g(\xv) = 0$ for all $\xv$ outside the $\ell_2$-ball of radius $w$ centered at the origin;
        \item $g(\xv) \in [0,2\epsilon]$ for all $\xv$, and $g(\bzero) = 2\epsilon$.
        \item $\| g \|_k \le c_1\frac{2\epsilon}{h(0)} \big(\frac{1}{w}\big)^{\nu} \|h\|_k$ when $k$ is the Mat\'ern-$\nu$ kernel on $\RR^d$, where $c_1$ is constant.  In particular, we have $\| g \|_k \le B$ when $w = \big( \frac{ 2\epsilon c_1 \|h\|_k }{ h(\bzero) B } \big)^{1/\nu}$.
    \end{itemize}
\end{lem}

This function can be used to simplify both the original analysis in \cite{Sca17a}, and the alternative proof in Sections \ref{sec:ana_std_simple}--\ref{sec:ana_std_cumul}.  We focus on the latter, and on the simple regret; the cumulative regret can be handled similarly.

We consider functions $\{f_1,\dotsc,f_M\}$ constructed similarly to Section \ref{sec:prelim}, but with each $f_m$ being a shifted and scaled version of $g(\xv)$ in Lemma \ref{lem:simpler_function}, using the choice of $w$ in the third statement of the lemma.  By the first part of the lemma, the functions have disjoint support as long as their center points are separated by at least $w$.  We also use $\frac{B}{3}$ in place of $B$ in the same way as Section \ref{sec:ana_std_simple}.  By forming a regularly spaced grid in each dimension, it follows that we can form
\begin{equation}
    M = \Big\lfloor \frac{1}{w} \Big\rfloor^d = \Big\lfloor  \frac{ h(\bzero) B }{ 6\epsilon c_1 \|h\|_k }\Big\rfloor^{d/\nu}
\end{equation}
such functions.\footnote{We can seemingly fit significantly more points using a sphere packing argument (e.g., \citep[Sec.~13.2.3]{DuchiNotes}), but this would only increase $M$ by a constant factor depending on $d$, and we do not attempt to optimize constants in this paper.}  Observe that this matches the $O\big( \big(\frac{B}{\epsilon}\big)^{d/\nu} \big)$ scaling in \eqref{eq:M_matern}.

We clearly still have the property that any point $\xv$ is $\epsilon$-optimal for at most one function.  The additional useful property here is that any point $\xv$ yields {\em any non-zero value} for at most one function.  Letting $\{\Rc_j\}_{j=1}^M$ be the partition of the domain induced by the above-mentioned grid (so that each $f_m$'s support is a subset of $\Rc_m$), we notice that \eqref{eq:std_simple_pf2} still holds (with the choice of function in the definition of $v^{j}_{m'}$ suitably modified), but now simplifies to
\begin{equation}
    \EE_{m}[N_{m'}(\tau)] \cdot (v^{m'}_{m'})^2 \ge \frac{\sigma^2}{2} \log\frac{1}{2.4 \delta}, \label{eq:std_alt_pf1}
\end{equation}
since there is no difference between $f(\xv) = f_m(\xv)$ and $f'(\xv) = f_m(\xv) + 2f_{m'}(\xv)$ outside region $\Rc_{m'}$.

Since the maximum function value is $2\epsilon$, we have $(v^{m'}_{m'})^2 \le 4\epsilon^2$, and substituting into \eqref{eq:std_alt_pf1} and summing over $m' \ne m$ gives $T \ge \frac{\sigma^2 (M-1)}{8\epsilon^2} \log\frac{1}{2.4\delta}$.  This matches \eqref{eq:std_simple_pf5} up to modified constant factors, but is proved via a simpler analysis.

\subsection{Robust Setting -- Corrupted Samples} \label{sec:corr_setting}

The high-level ideas behind proving Theorem \ref{thm:corr_lower} are outlined as follows, with the details in Appendix \ref{app:pf_corr}:
\begin{itemize}[leftmargin=5ex,itemsep=0ex,topsep=0.2ex]
    \item We consider an adversary that pushes all function values down to zero until its budget is exhausted.
    \item The function class chosen is similar to that in Figure \ref{fig:func_class_standard}, and we observe that (i) the adversary does not utilize much of its budget unless the sampled point is near the function's peak, and (ii) as long as the adversary is still active, the regret incurred at each time instant is typically $O(\epsilon)$ (since the algorithm has only observed $y_1=\dotsc=y_t = 0$, and thus has not learned where the peak is).
    \item We choose $\epsilon$ in a manner such that the adversary is still active at time $T$ for at least half of the functions in the class, yielding $R_T = \Omega(T \epsilon)$.  With $M$ functions in the class, we show that occurs when $T\epsilon = \Theta(CM)$.
    \item Combining $T\epsilon = \Theta(CM)$ with the choices of $M$ in \eqref{eq:M_se} and \eqref{eq:M_matern} yields the desired result.
\end{itemize}

\subsection{Robust Setting -- Corrupted Final Point} \label{sec:adv_robust_setting}
\vspace*{1ex}

To prove Theorem \ref{thm:lower_robust_new}, we introduce a new function class that overcomes the limitation of that of \cite{Bog18} (illustrated in Figure \ref{fig:func_class_robust} in Appendix \ref{app:existing}) in only handling success probabilities very close to one.  Here we only present an idealized version of the function class that cannot be used directly due to yielding infinite RKHS norm.  In Appendix \ref{app:pf_adv}, we provide the proof details and the precise function class used.
The idealized function class is depicted in Figure \ref{fig:func_class_new_robust} for both $d=1$ and $d=2$.

We consider a class of functions of size $M+1$, denoted by $\{f_0,f_1,\dotsc,f_M\}$.  For every function in the class, most points are within distance $\xi$ of a point with value $-2\epsilon$.  However, there is a narrow region (depicted in plain color in Figure \ref{fig:func_class_new_robust}) where this may not be the case.  The functions $f_1,\dotsc,f_M$ are distinguished only by the existence of one additional narrow spike going down to $-4\epsilon$ in this region (see the 1D case in Figure \ref{fig:func_class_new_robust}), whereas for the function $f_0$, the spike is absent.  For instance, in the 1D case, if the narrow spike has width $w'$, then the number of functions is $M+1 = \frac{\xi}{w'}+1$.

With this class of functions, we have the following crucial observations on when the algorithm returns a point with $\xi$-stable regret at most $\epsilon$:
\begin{itemize}[leftmargin=5ex,itemsep=0ex,topsep=0.2ex]
    \item Under $f_0$, the returned point $\xv^{(T)}$ must lie within the plain region of diameter $\xi$;
    \item Under any of $f_1,\dotsc,f_M$, the returned point $\xv^{(T)}$ must lie outside that plain region;
    \item The only way to distinguish between $f_0$ and a given $f_i$ is to sample within the associated narrow spike in which the function value is $-4\epsilon$.
\end{itemize}

Due to the $N(0,\sigma^2)$ noise, this roughly amounts to needing to take $\Omega\big( \frac{\sigma^2}{\epsilon^2} \big)$ samples within the narrow spike, and since there are $M$ possible spike locations, this means that $\Omega\big( \frac{M\sigma^2}{\epsilon^2} \big)$ samples are needed.  We therefore have a similar lower bound to \eqref{eq:std_simple_pf5}, and a similar regret bound to the standard setting follows (with modified constants additionally depending on $\xi$).

\begin{figure}
    \begin{centering}
        \includegraphics[width=0.48\textwidth]{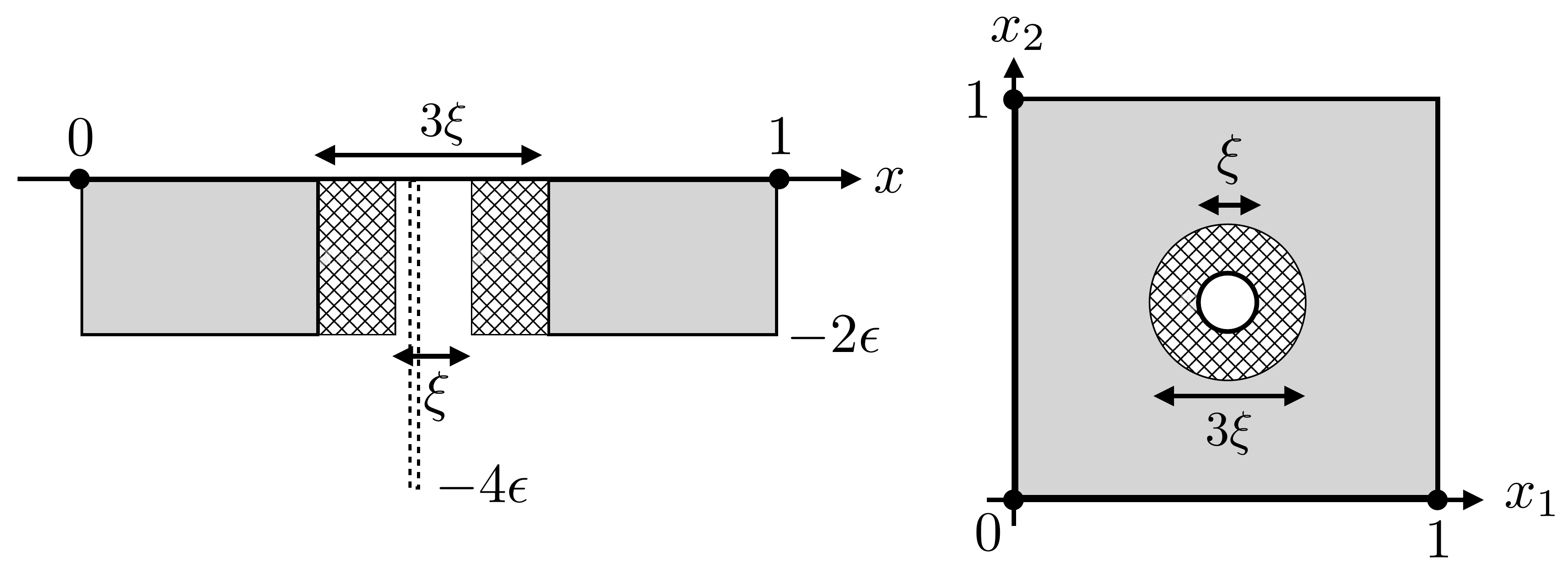}
        \par
    \end{centering}
    
    \caption{Idealized version of the function class under corrupted final points, in 1D (left) and 2D (right).  The shaded regions have value $-2\epsilon$; the checkered regions have value $0$ but their points can be perturbed into the shaded region; and the plain regions have value $0$ and cannot be.  If another spike is present, as per the dashed curve in the 1D case, then this creates a region (covering the entire plain region) that can be perturbed down to $-4\epsilon$. \label{fig:func_class_new_robust}} 
\end{figure}

\section{Conclusion}

We have provided novel techniques and results for algorithm-independent lower bounds in non-Bayesian GP bandit optimization.  In the standard setting, we have provided a new proof technique whose benefits include simplicity, versatility, and improved dependence on the error probability.

In the robust setting with corrupted samples, we have provided the first lower bound characterizing {\em joint} dependence on the corruption level and time horizon.  In the robust setting with a corrupted final point, we have overcome a limitation of the existing lower bound, demonstrating the impossibility of attaining any non-trivial constant error probability rather than only values close to one.  

An immediate direction for future work is to further close the gaps in the upper and lower bounds, particularly in the robust setting with corrupted samples.


\section*{Acknowledgement} 
This work was supported by the Singapore National Research Foundation (NRF) under grant number R-252-000-A74-281.

\newpage
\bibliographystyle{icml2021}
\bibliography{JS_References.bib}

\newpage
\onecolumn
\appendix

{\centering
    {\huge \bf Supplementary Material}
    
    {\Large \bf On Lower Bounds for Standard and Robust Gaussian \\ [1.5mm] Process Bandit Optimization (ICML 2021)}
    \par
}

\section{Formal Statements of Existing Results} \label{app:existing}

In this section, we provide a more detailed overview of existing regret bounds in the literature. While our main focus is on lower bounds, we also state several existing upper bounds for comparison purposes.  A particularly well-known upper bound from \cite{Sri09} is expressed in terms of the {\em maximum information gain}, defined as follows:
\begin{equation}
    \gamma_t = \max_{\xv_1, \cdots, \xv_t} \frac{1}{2} \log \det (\Iv_t + \sigma^{-2} \Kv_t), \label{eq:gamma_def}
\end{equation}
where $\Kv_t$ is a $t \times t$ kernel matrix with $(i,j)$-th entry $k(\xv_i,\xv_j)$.  It was established in \cite{Sri09} that  $\gamma_T = O( (\log T)^{d +1} )$ for the SE kernel, and $\gamma_T= O( T^{ \frac{d(d+1)}{ 2\nu + d(d+1) } } \log T )$  for the Mat\'ern-$\nu$ kernel, and we outline recent improvements on these bounds in Section \ref{sec:improved_matern}.

\subsection{Standard Setting}

We first state a standard cumulative regret upper bound \cite{Sri09,Cho17} and its straightforward adaptation to simple regret (e.g., see \citep[App.~C]{Bog20}).

\begin{thm} \label{thm:simple_ub}
    \emph{(Simple Regret Upper Bound -- Standard Setting \cite{Sri09,Cho17})}
    Fix $\epsilon > 0$, $B > 0$, $T \in \ZZ$, and $\delta \in (0,1)$, and suppose that
    \begin{equation}
        \frac{T}{\beta_T \gamma_T} \geq \frac{C_1}{\epsilon^2}, \label{eq:std_simple_general}
    \end{equation}
    where $C_1 = 8 / \log(1 + \sigma^{-2})$ and $\beta_T = \big(B + \sigma\sqrt{2(\gamma_{T-1} + \log\frac{e}{\delta})}\big)^2$.  Then, there exists an algorithm that, for any $f \in \Fc_k(B)$, returns $\xv^{(T)}$ satisfying $r(\xv^{(T)}) \le \epsilon$ with probability at least $1-\delta$.
\end{thm}

\begin{thm} \label{thm:cumul_ub}
    \emph{(Cumulative Regret Upper Bound -- Standard Setting \cite{Sri09,Cho17})}
    Fix $B > 0$, $T \in \ZZ$, and $\delta \in (0,1)$, and let $C_1 = 8 / \log(1 + \sigma^{-2})$ and $\beta_T = \big(B + \sigma\sqrt{2(\gamma_{T-1} + \log\frac{e}{\delta})}\big)^2$.  Then, there exists an algorithm such that, for any $f \in \Fc_k(B)$, we have $R_T \le \sqrt{C_1 T \beta_T \gamma_T}$ with probability at least $1-\delta$.
\end{thm}

The following lower bounds were proved in \cite{Sca17a}.

\begin{thm} \label{thm:simple_lb}
    \emph{(Simple Regret Lower Bound -- Standard Setting \citep[Thm.~1]{Sca17a})}
    Fix $\epsilon \in \big(0,\frac{1}{2}\big)$, $B > 0$, and $T \in \ZZ$.   Suppose there exists an algorithm that, for any $f \in \Fc_k(B)$, achieves average simple regret $\EE[r(\xv^{(T)})] \le \epsilon$.  Then, if $\frac{\epsilon}{B}$ is sufficiently small, we have the following:
    \begin{enumerate}
        \item For \emph{$k = \kSE$}, it is necessary that
        \begin{equation}
            T = \Omega\bigg( \frac{\sigma^2}{\epsilon^2} \Big(\log\frac{B}{\epsilon}\Big)^{d/2} \bigg). \label{eq:inst_se}
       \end{equation}
        \item For \emph{$k = \kMat$}, it is necessary that
        \begin{equation}
            T = \Omega\bigg( \frac{\sigma^2}{\epsilon^2} \Big(\frac{B}{\epsilon}\Big)^{d/\nu} \bigg). \label{eq:inst_Mat}
       \end{equation}
    \end{enumerate}
    Here, the implied constants may depend on $(d,l,\nu)$.
\end{thm}

\begin{thm} \label{thm:cumul_lb}
    \emph{(Cumulative Regret Lower Bound -- Standard Setting \citep[Thm.~2]{Sca17a})}
    For fixed $T \in \ZZ$ and $B > 0$, given any algorithm, we have the following:
    \begin{enumerate}
        \item For $k = \kSE$, there exists $f \in \Fc_k(B)$ such that
        \begin{equation}
            \EE[R_T] = \Omega\Bigg( \sqrt{T\sigma^2 \Big(\log \frac{B^2 T}{\sigma^2} \Big)^d } \Bigg) \label{eq:cumul_SE}
       \end{equation}
       provided that $\frac{\sigma}{B} = O\big(\sqrt{T}\big)$ with a sufficiently small implied constant.
        \item For $k = \kMat$, there exists $f \in \Fc_k(B)$ such that
        \begin{equation}
            \EE[R_T] = \Omega\bigg( B^{\frac{d}{2\nu + d}} \sigma^{\frac{2\nu}{2\nu + d}} T^{\frac{\nu + d}{2\nu + d}} \bigg) \label{eq:cumul_Mat}
       \end{equation}
       provided that $\frac{\sigma}{B} = O\big(\sqrt{T^{\frac{1}{2+d/\nu}}}\big)$ with a sufficiently small implied constant.
    \end{enumerate}
    Here, the implied constants may depend on $(d,l,\nu)$.
\end{thm}

\begin{rem}
    While Theorems \ref{thm:simple_lb} and \ref{thm:cumul_lb} are stated in terms of the average regret, it is also noted in \citep[Sec.~5.4]{Sca17a} that the same scaling laws hold for regret bounds that are required to hold with a fixed constant probability above $\frac{3}{4}$.  However, even when this probability is taken to approach one, the scaling of the lower bound therein remains unchanged, i.e., the dependence on the error probability is not characterized.  We provide refined bounds characterizing this dependence in Theorems \ref{thm:simple_lb_new} and \ref{thm:cumul_lb_new}.
\end{rem}

The function class used in the proofs of the above results is illustrated in Figure \ref{fig:func_class_standard}.  As discussed in \cite{Sca17a}, the upper and lower bounds are near-matching for the SE kernel, only differing in the constant multiplying $d$ in the exponent.  The gaps are more significant for the Mat\'ern kernel when relying on the bounds on $\gamma_T$ from \cite{Sri09}; however, in Section \ref{sec:improved_matern}, we overview some recent improved upper bounds that significantly close these gaps.

\subsection{Robust Setting -- Corrupted Samples}

In the robust setting with corrupted samples described in Section \ref{sec:corr_setting}, the following results were proved in \cite{Bog20}.

\begin{thm}
    \label{thm:known_C_ub}
    \emph{(Upper Bound -- Corrupted Samples \citep[Thm.~5]{Bog20})}
    In the setting of corrupted samples with corruption threshold $C$, RKHS norm bound $B$, and time horizon $T$, there exists an algorithm (assumed to have knowledge of $C$) that, with probability at least $1 - \delta$, attains cumulative regret
        $R_T = \mathcal{O}\big(\big(B + C + \sqrt{\ln(1 / \delta)}\big) \sqrt{\gamma_T T} + \gamma_{T}\sqrt{T} \big).$
\end{thm}

\begin{thm}
\label{thm:known_C_lb}
    \emph{(Lower Bound -- Corrupted Samples \citep[App.~J]{Bog20})}
    In the setting of corrupted samples with corruption threshold $C$, if the RKHS norm $B$ exceeds some universal constant, then for any algorithm, there exists a function $f \in \Fc_k(B)$ that incurs $\Omega(C)$ cumulative regret with probability arbitrarily close to one for any time horizon $T \ge C$.
\end{thm}

Note that Theorem \ref{thm:known_C_ub} concerns the case that $C$ is known.  Additional upper bounds for the case of unknown $C$ are also given in \cite{Bog20}, but we focus on the known $C$ case; this is justified by the fact that we are focusing on lower bounds, and any given lower bound is stronger when it also applies to algorithms knowing $C$.  Having said this, in future work, it may be interesting to determine whether the case of unknown $C$ is provably harder; this is partially addressed in \cite{Bog20a} in the linear bandit setting.

Although Theorem \ref{thm:known_C_lb} shows that the linear dependence on $C$ is unavoidable, characterizing the optimal {\em joint} dependence on $C$ and $T$ is very much an open problem, as highlighted in \cite{Bog20}.  Letting $\Rbar^{(0)}_T$ and $\Runder^{(0)}_T$ be generic upper and lower bounds on the cumulative regret in the uncorrupted setting, we see that Theorem \ref{thm:known_C_ub} is of the form $O\big( C \Rbar^{(0)}_T \big)$ (multiplicative dependence on $C$), whereas Theorem \ref{thm:known_C_lb} implies a lower bound of $\Omega\big( \Runder^{(0)}_T + C \big)$ (additive dependence on $C$).

A similar gap briefly existed in the standard multi-armed bandit problem \cite{Lyk18}, but was closed in \cite{Gup19}, in which the additive dependence was shown to be tight.  However, the techniques for attaining a matching upper bound do not appear to extend easily to the RKHS setting.  In Section \ref{sec:corr_setting}, we show that, at least for deterministic algorithms, a fully additive dependence is in fact impossible.

\subsection{Robust Setting -- Corrupted Final Point} \label{sec:adv_robust_results}

In the robust setting with a corrupted final point described in Section \ref{sec:adv_robust_setting}, the following results were proved in \cite{Bog18}.

\begin{thm} \label{thm:upper}
    \emph{(Upper Bound -- Corrupted Final Point \citep[Thm.~1]{Bog18})}
    Fix $\xi > 0$, $\epsilon > 0$, $B > 0$, $T \in \ZZ$, $\delta \in (0,1)$, and a distance function $\dist(\xv,\xv')$, and suppose that
    \begin{equation}
        \frac{T}{\beta_T \gamma_T} \geq \frac{C_1}{\epsilon^2}, \label{eq:adv_rob_general}
    \end{equation}
    where $C_1 = 8 / \log(1 + \sigma^{-2})$ and $\beta_T = \big(B + \sigma\sqrt{2(\gamma_{T-1} + \log\frac{e}{\delta})}\big)^2$.  Then, there exists an algorithm that, for any $f \in \Fc_k(B)$, returns $\xv^{(T)}$ satisfying $r_{\xi}(\xv^{(T)}) \le \epsilon$ with probability at least $1-\delta$.
\end{thm}

\begin{thm} \label{thm:lower_robust}
    \emph{(Lower Bound -- Corrupted Final Point \citep[Thm.~2]{Bog18})}
    Fix $\xi \in \big(0,\frac{1}{2}\big)$, $\epsilon \in \big(0,\frac{1}{2}\big)$, $B > 0$, and $T \in \ZZ$, and set $\dist(\xv,\xv') = \|\xv - \xv'\|_2$.   Suppose that there exists an algorithm that, for any $f \in \Fc_k(B)$, reports a point $\xv^{(T)}$ achieving $\xi$-regret $r_{\xi}(\xv^{(T)}) \le \epsilon$ with probability at least $1 - \delta$.  Then, provided that $\frac{\epsilon}{B}$ and $\delta$ are sufficiently small, we have the following:
    \begin{enumerate}[leftmargin=5ex,itemsep=0ex,topsep=0.2ex]
        \item For \emph{$k = \kSE$}, it is necessary that
        $T = \Omega\big( \frac{\sigma^2}{\epsilon^2} \big(\log\frac{B}{\epsilon}\big)^{d/2} \big)$. 
        \item For \emph{$k = \kMat$}, it is necessary that
        $T = \Omega\big( \frac{\sigma^2}{\epsilon^2} \big(\frac{B}{\epsilon}\big)^{d/\nu} \big).$ 
    \end{enumerate}
    Here, the implied constants may depend on $(\xi,d,l,\nu)$.
\end{thm}

For the SE kernel, \eqref{eq:adv_rob_general} holds with $T = \Otil\big( \frac{1}{\epsilon^2} \big( \log\frac{1}{\epsilon} \big)^{2d} \big)$ for constant $B$ and $\sigma^2$ \cite{Bog18}, where $\Otil(\cdot)$ hides dimension-independent log factors.  Thus, the upper and lower bounds nearly match.  A detailed treatment of the Mat\'ern kernel is deferred to Appendix \ref{sec:improved_matern}.

The function class used in the proof of Theorem \ref{thm:lower_robust} is illustrated in Figure \ref{fig:func_class_robust}.  In contrast to the standard setting, where the difficulty of the class used (depicted in Figure \ref{fig:func_class_standard}) was in narrowing down the main (positive) peak, here the difficulty is in {\em avoiding} any point that can be perturbed into a {\em negative} valley.

While such an approach suffices for proving Theorem  \ref{thm:lower_robust}, it has the significant drawback that an algorithm that returns a {\em completely random} point (even with $T = 0$) has a fairly high chance of being within $\epsilon$ of optimal.  That is, Theorem  \ref{thm:lower_robust} only gives a hardness result for the case that the algorithm is required to succeed with probability sufficiently close to one (i.e., $\delta$ is sufficiently small).  

This problem is exacerbated further in higher dimensions and/or for smaller values of $\xi$.  To see this, note that the ``bad'' region (gray area in Figure \ref{fig:func_class_robust}) has volume proportional to $\xi^d$, which may be very close to zero (whereas the volume of the domain $[0,1]^d$ is one for any $d$).  In light of this limitation, it would be preferable to have a hardness result associated with any algorithm that succeeds with a universal constant probability, rather than only those that succeed with very high probability depending on $\xi$ and $d$.  In Section \ref{sec:adv_robust_setting}, we present a refined bound that addresses this exact issue.

\begin{figure}
    \begin{centering}
        \includegraphics[width=0.5\textwidth]{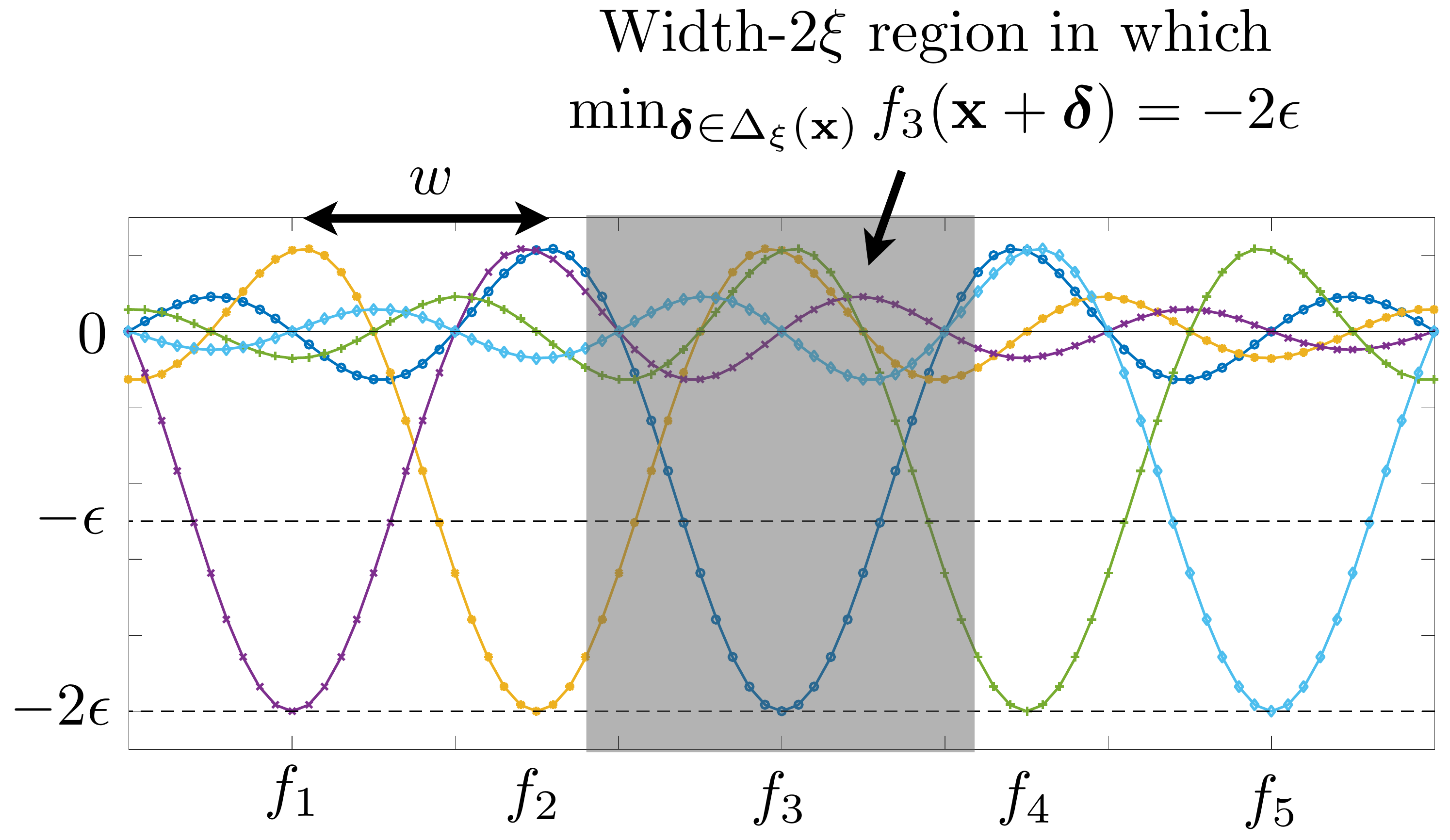}
        \par
    \end{centering}
    
    \caption{Illustration of functions $f_1,\dotsc,f_5$ equal to a common function shifted by various multiples of a given parameter $w$.  In the $\xi$-stable setting, there is a wide region (shown in gray for the dark blue curve $f_3$) within which the perturbed function value equals $-2\epsilon$. \label{fig:func_class_robust}}
\end{figure}

\subsection{Further Existing Upper Bounds for the Mat\'ern Kernel} \label{sec:improved_matern}

When comparing the lower bounds from \cite{Sca17a} with the upper bounds from \cite{Sri09}, the gaps are relatively small for the SE kernel, e.g., $\Otil\big( \sqrt{T (\log T)^{2d}} \big)$ \cite{Sri09} vs.~$\Omega\big( \sqrt{T (\log T)^{d/2}} \big)$ \cite{Sca17a} for the cumulative regret.\footnote{The published version of \cite{Sca17a} mistakenly omits the division by two in the exponent; see \url{https://arxiv.org/abs/1706.00090v3} for a corrected version.} In contrast, the gaps are more significant for the Mat\'ern kernel, e.g., $\Otil\big(T^{\frac{1}{2}\cdot\frac{2\nu+3d(d+1)}{2\nu+d(d+1)}}\big)$ \cite{Sri09} vs.~$\Otil\big(T^{\frac{\nu+d}{2\nu+d}}\big)$ \cite{Sca17a}, with the former in fact failing to be sub-linear in $T$ unless $2\nu-d(d+1)>0$.  In the following, we outline some more recent results and observations that significantly close these gaps for the Mat\'ern kernel. We focus our discussion on the cumulative regret (with constant values of $B$ and $\sigma^2$), but except where stated otherwise, similar observations apply in the case of simple regret.

Recently, \cite{Jan20} gave the first practical algorithm (referred to as $\pi$-GP-UCB) to provably attain sublinear regret for all $\nu > 1$ and $d \ge 1$ in the RKHS setting.  Specifically, the regret bound is $\Otil\big( T^{\frac{d(2d+3)+2\nu}{d(2d+4)+4\nu}} \big)$. 
It was also pointed out in \cite{Jan20} that the SupKernelUCB algorithm of \cite{Val13} can be extended to continuous domains via a fine discretization; the dependence on the number of arms $N$ is logarithmic, so even a very fine discretization of the space with $\frac{1}{{\rm poly}(T)}$ spacing only amounts to a $\log N = O(d \log T)$ term.  With this approach, one attains a cumulative regret of $R_T = \Otil(\sqrt{T\gamma_T})$ (assuming $d$ to be constant); in contrast with the result in \cite{Sri09}, this yields $R_T = o(T)$ whenever $\gamma_T = o(T)$ (or more precisely, whenever $\gamma_T$ is sufficiently sublinear to overcome any hidden logarithmic factors).
Despite its significant theoretical value, SupKernelUCB has been observed to perform poorly in practice \cite{Cal19}.

Another recent work \cite{She20} gave an algorithm whose regret bounds further improve on those of \cite{Jan20}.  The algorithm is again impractical due to the large constant factors, though a practical heuristic is also given.  Unlike the other works outlined above, the simple regret and cumulative regret behave very differently, due to the exploratory nature of the algorithm.  Defining $\Ic_0 = (0,1]$, $\Ic_1 = \big(1,\frac{d(d+1)}{2}\big]$, $\Ic_2 = \big(\frac{d(d+1)}{2},\frac{d^2+5d+12}{4}\big]$, and $\Ic_3 = (0,\infty) \setminus (\Ic_0 \cup \Ic_1 \cup \Ic_2)$, the simple regret bounds of \cite{She20} are summarized as follows:
\begin{itemize}
    \item For $\nu \in \Ic_0 \cup \Ic_1$ (i.e., $\nu \le \frac{d(d+1)}{2}$), one has $r(\xv^{(T)}) = \Otil\big( T^{\frac{-\nu}{2\nu + d}} \big)$, which matches the lower bound of \cite{Sca17a};
    \item For $\nu \in \Ic_2$ (which is only possible for $d \le 5$), one has $r(\xv^{(T)}) = \Otil\big( T^{\frac{-1}{d + 2}} \big)$;
    \item For $\nu \in \Ic_3$, one has $r(\xv^{(T)}) = \Otil\big( T^{-\frac{1}{2} + \frac{d(d+3)}{4\nu+d(d+5)}} \big)$.
\end{itemize}
In addition, the cumulative regret bounds of \cite{She20} are summarized as follows:
\begin{itemize}
    \item For $\nu \in \Ic_0$ (i.e., $\nu \le 1$), one has $R_T = \Otil( T^{\frac{\nu+d}{2\nu + d}} )$, which matches the lower bound of \cite{Sca17a}.
    \item For $\nu \in \Ic_1 \cup \Ic_2$, one has $R_T = \Otil( T^{\frac{d+1}{d+2}} )$;
    \item For $\nu \in \Ic_3$, one has $R_T = \Otil\big( T^{\frac{1}{2} + \frac{d(d+3)}{4\nu+d(d+5)}} \big)$.
\end{itemize}
Note that all of the preceding bounds are high-probability bounds (e.g., holding with probability $0.99$).

Finally, in a very recent work \cite{Vak20a}, improved bounds on the information gain $\gamma_T$ were given, in particular yielding $\gamma_T= \Otil\big(T^{\frac{d}{2\nu + d}}\big)$ for the Mat\'ern kernel with $\nu > \frac{1}{2}$.  When combined with the above-mentioned $R_T = \Otil(\sqrt{T\gamma_T})$ upper bound, this in fact yields matching upper and lower regret bounds for the Mat\'ern kernel, up to logarithmic factors.  

With the above outline in place, we are now in a position to explain the Mat\'ern kernel upper bounds shown in Table \ref{tbl:standard}:
\begin{itemize}
    \item For the standard setting, we first apply the above-mentioned upper bound $R_T = \Otil(\sqrt{T\gamma_T})$ for SupKernelUCB, which also has a multiplicative $\sqrt{\log\frac{1}{\delta}}$ dependence on the target error probability $\delta$ \cite{Val13}.  Substituting $\gamma_T= \Otil\big(T^{\frac{d}{2\nu + d}}\big)$ gives the desired bound, $R_T = O^{*}\big(T^{\frac{\nu+d}{2\nu+d}} \sqrt{\log\frac{1}{\delta}}\big)$.
    \item For the setting of corrupted samples, we substitute $\gamma_T= \Otil\big(T^{\frac{d}{2\nu + d}}\big)$ and $\delta = \Theta(1)$ into Theorem \ref{thm:known_C_ub}.  Notice that the term containing $C$ is only multiplied by $\sqrt{\gamma_T}$, rather than $\gamma_T$.
    \item For the setting of a corrupted final point, we first substitute $\gamma_T= \Otil\big(T^{\frac{d}{2\nu + d}}\big)$ into Theorem \ref{thm:upper}.  We then note that $\beta_T = O(\max\{ \gamma_T, \log\frac{1}{\delta} \})$, and treat two cases separately depending on which term attains the maximum. The case that $\beta_T = O(\gamma_T)$ leads to the term $O^{*}\big(\big(\frac{1}{\epsilon}\big)^{\frac{2(2\nu+d)}{2\nu-d}}\big)$ (and requires $d < 2 \nu$ to obtain a non-vacuous statement in \eqref{eq:adv_rob_general}), and the case that $\beta_T = O\big( \log\frac{1}{\delta} \big)$ leads to the term $O^{*}\big(\big(\frac{\log\frac{1}{\delta}}{\epsilon^{2}}\big)^{1+\frac{d}{2\nu}}\big)$.
\end{itemize}
Finally, we briefly note that we could similarly slightly improve the SE kernel upper bounds containing $\delta$ in Table \ref{tbl:standard} by treating two cases separately similarly to last item above; however, for the SE kernel this only amounts to minor differences.  Specifically, the standard cumulative regret from \cite{Cho17} is $O^{*}\big(\sqrt{T(\log T)^{2d} + T (\log T)^d \log\frac{1}{\delta}} \big)$, and the time to $\epsilon$-optimality from \cite{Bog18} is $O^{*}\big(\frac{1}{\epsilon^{2}}\big(\log\frac{1}{\epsilon}\big)^{2d} + \frac{1}{\epsilon^{2}}\big(\log\frac{1}{\epsilon}\big)^{d} \log\frac{1}{\delta}\big)$.

\subsection{Other Settings} \label{sec:other}


The following works are more distinct from ours, so we only give a brief outline:
\begin{itemize}
    \item The noiseless setting with RKHS functions was studied in \cite{Bul11,Lyu19,Vak20}.  In particular, \cite{Bul11} characterized the minimax-optimal simple regret for the Mat\'ern kernel, showing that $O\big( \big(\frac{1}{\epsilon}\big)^{d/\nu} \big)$ is the optimal time to $\epsilon$-optimality.  The papers \cite{Lyu19,Vak20} also considered the cumulative regret, which can be considerably smaller compared to the noisy setting (e.g., \cite{Vak20} shows that $O(1)$ cumulative regret is possible when $\nu > d$).
    \item A counterpart to the non-Bayesian RKHS setting is the Bayesian setting, in which $f$ is assumed to be randomly drawn from a Gaussian process.  Regret bounds for such settings have also been extensively studied, and the regret can again be much smaller in the noiseless setting \cite{Gru10,Def12,Kaw15,Gri18}.  In the noisy setting, the upper bounds of \cite{Sri09} are dictated by $\sqrt{T \gamma_T}$, in contrast with the higher $\sqrt{T \gamma_T^2}$ term in the RKHS setting.  Further improvements over \cite{Sri09} are given in \cite{She17,Sca18a,Wan20}, including broad scenarios in which the order-optimal scaling of the cumulative regret is roughly $\sqrt T$.
    \item Follow-up works to \cite{Bog18} include (i) investigations of mixed strategies \cite{Ses20} with an average function value requirement, and (ii) distributional robustness in place of the simple perturbation model \cite{Kir20,Ngu20}.  In addition, earlier works considered related notions of {\em input noise} in Bayesian optimization \cite{Nog16,Bel17,Dai17}.  Other robustness notions in GP optimization include outliers \cite{Mar18}, batch settings with failed experiments \cite{Bog18a}, and heavy-tailed noise \cite{Cho19}.  The latter of these provides lower bounds adapted from the techniques of \cite{Sca17a}, indicating that our alternative techniques based on Lemma \ref{lem:relating} could also be of interest in such a setting.
\end{itemize}

\section{Proof of Theorem \ref{thm:corr_lower} (Corrupted Samples)} \label{app:pf_corr}

The analysis re-uses some aspects of the proof of Theorem \ref{thm:cumul_lb_new} for the standard setting, given in Section \ref{sec:proofs}.  We consider the function class $\Fc = \{f_1,\dotsc,f_M\}$ from Section \ref{sec:prelim}, with a parameter $\epsilon > 0$ to be chosen later.  As mentioned in Section \ref{sec:setup_corr}, for a deterministic algorithm, the adversary knows which action is played each round.  Hence, we can consider an adversary that trivially pushes the function value $f(\xv_t)$ down to zero, at a cost of $|c_t(\xv_t)| = f(\xv_t)$, until its budget does not allow doing so.  When the remaining budget does not allow pushing the function value down to zero, the adversary pushes the value as close to zero as possible, after which its budget is exhausted and no further corruptions occur.

Since we are considering the noiseless setting, the player simply observes $y_1 = y_2 = \dotsc = y_t = 0$ for any time $t$ before the adversary exhausts its budget.  In particular, if the adversary still has not exhausted its budget by time $T$, then we have $y_1 = y_2 = \dotsc = y_T = 0$.  Intuitively, in this case, since the player has not observed any function values, it cannot know where the function peak is.  Since any sample away from the peak incurs regret $\Theta(\epsilon)$ by construction, this leads to $\Omega(T\epsilon)$ regret.  We note that ideas with similar intuition have also appeared in simpler robust bandit problems, including the finite-arm setting \citep[Sec.~5]{Lyk18} and the case of linear rewards \cite{Bog20a}.

To make this intuition precise, we first introduce the following terminology:  For a given set of sampled points $\xv_1,\dotsc,\xv_T$ up to time $T$, define a given function $\tilde{f} \in \Fc$ to be {\em corruptible after time $T$} if $\sum_{t=1}^T |\tilde{f}(\xv_t)| < C$, and {\em non-corruptible after time $T$} otherwise.   Then, we have the following lemma.

\begin{lem} \label{lem:corr_budget}
    Suppose that  $T = \frac{\alpha CM}{\epsilon}$ for some sufficiently small constant $\alpha > 0$.  Then, under the preceding setup, for any set of sampled points $\xv_1,\dotsc,\xv_T$, there exists a set of $\frac{M}{2}$ functions among $\{f_j\}_{j=1}^M$ that are corruptible after time $T$, i.e., $\sum_{t=1}^T |f_j(\xv_t)| < C$.
\end{lem}
\begin{proof}
    We will upper bound the average of $\sum_{t=1}^T |f_j(\xv_t)|$ over all values $j \in \{1,\dotsc,M\}$, and then use Markov's inequality to establish the existence of the $\frac{M}{2}$ values of $j$ in the lemma statement.  First, for any {\em fixed} time index $t$ and point $\xv_t$, the second part of Lemma \ref{lem:vbar_sums} implies that
    \begin{equation}
        \sum_{m=1}^M |f_m(\xv_t)| = O(\epsilon).
    \end{equation}
    Renaming $m$ to $j$, summing both sides over $t =1,\dotsc,T$, and dividing both sides by $M$, it follows that
    \begin{equation}
        \frac{1}{M} \sum_{j=1}^M \bigg( \sum_{t=1}^T |f_j(\xv_t)| \bigg) = O\Big( \frac{T\epsilon}{M} \Big) = O(\alpha C) \le \frac{C}{4}, \label{eq:avg_M}
    \end{equation}
    where the first equality uses the assumption $T = \frac{\alpha CM}{\epsilon}$, and the second equality uses the assumption that $\alpha$ is sufficiently small.  

    As hinted above, interpreting the left-hand side of \eqref{eq:avg_M} as an average of $\sum_{t=1}^T |f_m(\xv_t)|$ with respect to $j$, Markov's inequality implies that we can only have $\sum_{t=1}^T |f_j(\xv_t)| \ge C$ for at most a $\frac{1}{4}$ fraction of the $j$ values in $\{1,\dotsc,M\}$.  Thus, at least $\frac{M}{2}$ of the functions give $\sum_{t=1}^T |f_j(\xv_t)| < C$, as desired.

\end{proof}

Since we are considering deterministic algorithms, we have that when observing $y_1=0$, $y_2=0$, and so on, any algorithm can only follow a fixed corresponding sequence $\xv_1$, $\xv_2$, and so on (until a non-zero $y_t$ value is observed).  However, Lemma \ref{lem:corr_budget} implies that no matter which such fixed sequence is chosen, there exist $\frac{M}{2}$ functions under which the adversary is able to continue corrupting $y_1=y_2=\dotsc=y_T = 0$ up until the final point $T$.  Since the function class is such as that any given point is $\epsilon$-optimal for at most one function, it follows that the algorithm can only attain $R_T = o(T\epsilon)$ for at most one of these $\frac{M}{2}$ functions; all of the others must incur $R_T = \Omega(T\epsilon)$.



To make the $\Omega(T \epsilon)$ lower bound explicit, we need to select $\epsilon$ as a function of $T$.  However, such a choice must be consistent with two assumptions already made in the above analysis: (i) $T = \frac{\alpha CM}{\epsilon}$ for some sufficiently small $\alpha = \Theta(1)$ in Lemma \ref{lem:corr_budget}, and (ii) the function class in Section \ref{sec:prelim} is such that $M$ satisfies \eqref{eq:M_se} (SE kernel) or \eqref{eq:M_matern} (Mat\'ern) kernel.  Handling the two kernels separately, we proceed as follows:
\begin{itemize}
    \item For the SE kernel, substituting \eqref{eq:M_se} into $T = \frac{\alpha CM}{\epsilon}$ yields $T = \Theta\big( \frac{C}{\epsilon} \big(\log\frac{1}{\epsilon}\big)^{d/2} \big)$, and using the assumptions $\Theta(1) \le C \le T^{1-\Omega(1)}$ and $d = \Theta(1)$, an inversion of this expression (detailed in Section \ref{sec:inversion} below) gives $\epsilon = \Theta\big( \frac{C}{T} (\log T)^{d/2} \big)$.  Hence, we have $R_T = \Omega(T\epsilon) = \Omega\big( C (\log T)^{d/2} \big)$.
    \item For the Mat\'ern kernel, substituting \eqref{eq:M_matern} into $T = \frac{\alpha CM}{\epsilon}$ yields $T = \Theta\big( \frac{C}{\epsilon} \big(\frac{1}{\epsilon}\big)^{d/\nu} \big)$, and inverting this gives $\epsilon = \Theta\big( \big(\frac{C}{T}\big)^{\frac{1}{1+d/\nu}} \big) = \Theta\big( \big(\frac{C}{T}\big)^{\frac{\nu}{d+\nu}} \big)$.  Hence, we have $R_T = \Omega(T\epsilon) = \Omega\big( C^{\frac{\nu}{d+\nu}} T^{\frac{d}{d+\nu}} \big)$.
\end{itemize}
This completes the proof of Theorem \ref{thm:corr_lower}.  

We remark that since this proof considers the noiseless setting (i..e, $\sigma^2 = 0$), it may be interesting to establish whether the arguments can be refined for the noisy setting in order to obtain an improved lower bound on $R_T$.

\subsection{Final Inversion Step for the SE Kernel} \label{sec:inversion}

We first note that the scaling $T = \Theta\big( \frac{C}{\epsilon} \big(\log\frac{1}{\epsilon}\big)^{d/2} \big)$ is equivalent to
\begin{equation}
    \frac{T}{C} = \Theta\bigg( \frac{1}{\epsilon} \Big(\log\frac{1}{\epsilon}\Big)^{d/2} \bigg). \label{eq:two_sides}
\end{equation}  
We proceed by taking the logarithm on both sides.  For the left-hand side, the assumption $\Theta(1) \le C \le T^{1-\Omega(1)}$ implies $\log \frac{T}{C} = \Theta(\log T)$.  In addition, the assumption $d = \Theta(1)$ implies that the logarithm of the right-hand side of \eqref{eq:two_sides} behaves as $\Theta\big(\log\frac{1}{\epsilon} + \frac{d}{2}\log\log\frac{1}{\epsilon}\big) = \Theta\big( \log\frac{1}{\epsilon} \big)$ (note that $\epsilon \in (0,1)$, so $\frac{1}{\epsilon} > 1$).  Hence, overall, taking the logarithm on both sides of \eqref{eq:two_sides} gives $\log\frac{1}{\epsilon} = \Theta(\log T)$.  Substituting this finding into \eqref{eq:two_sides} gives $\frac{T}{C} = \Theta\big( \frac{1}{\epsilon} (\log T )^{d/2}  \big)$, and re-arranging gives $\epsilon = \Theta\big( \frac{C}{T} (\log T)^{d/2} \big)$ as claimed.

\section{Proof of Theorem \ref{thm:lower_robust_new} (Corrupted Final Point)} \label{app:pf_adv}

We continue the proof following the intuition provided for the idealized function class in Section \ref{sec:adv_robust_setting}.

\subsection{Details for the Mat\'ern Kernel}

We seek to provide a function class that captures the essential properties of the idealized version, while ensuring that every function in the class has RKHS norm at most $B$ under the Mat\'ern kernel.  Recall that Theorem \ref{thm:lower_robust_new} assumes that $\frac{\epsilon}{B}$ is sufficiently small.

To construct a given function $f_m$ of the form in Figure \ref{fig:func_class_new_robust}, we will decompose it as
\begin{equation}
    f_m(\xv) = -c(\xv) + b(\xv) - s_m(\xv), \label{eq:f_decomp}
\end{equation}
where $c(\cdot)$ is a constant function equaling $-2\epsilon$ across the whole domain $[0,1]^d$, $b(\cdot)$ approximates the indicator function (scaled by $2\epsilon$) of being within a ball ($d \ge 2$) or interval ($d=1$) of diameter roughly $3\xi$ at the center of the domain (see below for details), and $s_m(\xv)$ is the narrow spike whose location is determined by $m \in \{1,\dotsc,M\}$ (whereas $s_0(\xv) = 0$ for all $\xv$).  We proceed by showing that suitable functions can be constructed having RKHS norm at most $\frac{B}{3}$ each, so that the triangle inequality applied to \eqref{eq:f_decomp} yields $\|f_m\|_k \le B$.  

For convenience, we first work with auxiliary functions centered at the origin, before shifting them to be centered at a suitable point in $[0,1]^d$.  

\begin{lem} \label{lem:circ_func}
    Let $k$ be the Mat\'ern-$\nu$ kernel, let $r > 0$ and $0 < w_0 \le \frac{r}{2}$ be fixed constants, and let $\epsilon > 0$ and $B > 0$ be such that $\frac{\epsilon}{B}$ is sufficiently small.  There exists a function $b_0(\xv)$ on $\RR^d$ satisfying (i) $b_0(\xv) = 2\epsilon$ whenever $\|\xv\|_2 \le r-w_0$; (ii) $b_0(\xv) = 0$ whenever $\|\xv\|_2 \ge r+w_0$; (iii) $b_0(\xv) \in [0,2\epsilon]$ whenever $r-w_0 \le \|\xv\|_2 \le r+w_0$; and (iv) $\|b_0\|_k \le \frac{B}{3}$.
\end{lem}
\begin{proof}
    Define the auxiliary ``ball'' function with radius $r > 0$ as
    \begin{equation}
        {\rm ball}(\xv) = \openone\{ \|\xv\|^2 \le r \},
    \end{equation}
    and for fixed $w_0 > 0$, let $g_0(\xv) = h\big( \frac{\xv}{w_0} \big)$ be a scaled version of the bump function $h(\cdot)$ from Lemma \ref{lem:simpler_function}, and define
    \begin{equation}
        \tilde{b}_0(\xv) = (g_0 \star {\rm ball})(\xv),
    \end{equation}
    where $\star$ denotes the convolution operation.  By the definition of convolution and the fact that $g_0(\xv)$ is non-zero only for $\|\xv\|_2 \le w_0$, we have the following:
    \begin{itemize}
        \item For $\xv$ satisfying $\|\xv\|_2 \ge r+w_0$, we have $\tilde{b}_0(\xv) = 0$;
        \item For $\xv$ satisfying $\|\xv\|_2 \le r-w_0$, we have $\tilde{b}_0(\xv) = \int_{\RR^d} g_0(\xv){\rm d\xv}$, which is a constant depending on $w_0$.
        \item For $\xv$ satisfying $r-w_0 \le \|\xv\|_2 \le r+w_0$, we have that $\tilde{b}_0(\xv)$ equals some value in between the two constants given in the previous two dot points.
    \end{itemize}
    We proceed by showing that $\|\tilde{b}_0\|_k < \infty$ under the Mat\'ern kernel.  We know from the proof of Lemma \ref{lem:simpler_function} that $\|g_0\|_k$ is a finite constant depending on $w_0$.  As for ${\rm ball}(\cdot)$, it suffices for our purposes to note that its Fourier transform is bounded in absolute value (point-wise) by a constant depending on $r$, which is seen by writing
    \begin{equation}
        \Big| \int_{\RR^d} {\rm ball}(\xv) e^{\imath \langle \xv, \bxi \rangle } {\rm d}\xv \Big| \le \int_{\|\xv\|_2 \le r} {\rm d}\xv < \infty. \label{eq:bound_const0}
    \end{equation}
    Then, using the formula for RKHS norm in Lemma \ref{lem:rkhs_norm}, and using capital letters to denote the Fourier transforms of the respective spatial functions, we have
    \begin{align}
        \|\tilde{b}_0\|_{k} &= \int \frac{ |G_0(\bxi)|^2 \cdot |{\rm BALL}(\bxi)|^2  }{ K(\bxi) } d\bxi \label{eq:conv_prod} \\
            &\le O(1) \cdot \int \frac{ |G_0(\bxi)|^2 }{ K(\bxi) } d\bxi = O(1) \cdot \|g_0\|_k < \infty, \label{eq:bound_const}
    \end{align}
    where \eqref{eq:conv_prod} uses the fact that convolution in the spatial domain corresponds to multiplication in the Fourier domain, and \eqref{eq:bound_const} uses \eqref{eq:bound_const0}.

    Finally, for any fixed $r$ and $w_0$, we define $b_0(\xv)$ to be a constant times $\tilde{b}_0(\xv)$, with the constant chosen so that the maximum function value is $2\epsilon$.  Since $\|\tilde{b}_0\|_k = O(1)$ and we scale by $O(\epsilon)$, it follows that $\|b_0\|_k = O(\epsilon)$, and thus $\|b_0\|_k \le \frac{B}{3}$ due to the assumption that $\frac{\epsilon}{B}$ is sufficiently small.  The remaining properties of $b_0(\cdot)$ in the lemma statement are directly inherited from those of $\tilde{b}_0(\cdot)$ above.
\end{proof}

We now construct the functions in \eqref{eq:f_decomp} as follows for some arbitrarily small constant $\eta > 0$:
\begin{itemize}
    \item For $c(\xv)$, let $r = \sqrt{d} + \eta$ and $w_0 = \eta$, so that $c(\xv) = 1$ for all $\xv \in [0,1]^d$;
    \item Let $b(\xv)$ be a shifted version (to be centered at $\big(\frac{1}{2},\dotsc,\frac{1}{2}\big)$) of the ball function $b_0(\xv)$ with $r = (3 - \eta)\xi$ and $w_0 = \eta \xi$.
    \item For $m=1,\dotsc,M$, let $s_m(\xv)$ be a shifted version of the spike formed in Lemma \ref{lem:simpler_function}, with RKHS norm $\frac{B}{3}$ in place of $B$.
\end{itemize}
While the radius of the ``plain'' region in Figure \ref{fig:func_class_new_robust} (i.e., the region where points may have function value zero even after a worst-case perturbation) is not exactly $\xi$ due to the ``leeway'' introduced by $\eta$, it is arbitrarily close when $\eta$ is sufficiently small (e.g., $0.99\xi$). 

Using the assumption that $\frac{\epsilon}{B}$ is sufficiently small but $\xi$ is constant, the choice of $w$ in Lemma \ref{lem:simpler_function} means that we can assume $w$ to be much smaller than $\xi$ (e.g., a $0.1$ fraction or less).  As a result, a standard packing argument \citep[Sec.~13.2.3]{DuchiNotes} reveals that we can ``pack'' at least $M = \big( \frac{c'_0 \xi}{w} \big)^d$ bump functions into the sphere of radius roughly $\xi$ (for some absolute constant $c'_0$), while ensuring that the supports of these functions are non-overlapping.  Since $\xi$ is assumed to be constant, this choice of $M$ matches \eqref{eq:Mw} up to a possible change in the value of $c'_0$, and we conclude that the scaling \eqref{eq:M_matern} applies with a possibly different choice of $c_3$.

With this fact in place, we can proceed in the same way as Section \ref{sec:simplified_matern}.  The ``base function'' and ``auxiliary'' function are chosen as $f(\xv) = f_0(\xv)$ and $f'(\xv) = f_{m'}(\xv)$ (for some $m'=1,\dotsc,M$), so that their difference is $s_{m'}(\xv)$ (since $s_0(\xv) = 0$).  Using \eqref{eq:std_alt_pf1} with the substitution $v_{m'}^{m'} \leftarrow 4\epsilon$ (i.e., the maximal value of $s_{m'}(\xv)$) and $N_{m'}(\tau)$ redefined to be the number of samples within the support of $s_{m'}(\xv)$, we obtain
\begin{equation}
    \EE_{m}[N_{m'}(\tau)] \cdot (4\epsilon)^2 \ge \frac{\sigma^2}{2} \log\frac{1}{2.4 \delta},
\end{equation}
and summing over $m' =1,\dotsc,M$ gives
\begin{equation}
    T \ge \frac{\sigma^2 M}{32\epsilon^2} \log\frac{1}{2.4\delta}.
\end{equation}
Substituting the scaling on $M$ in \eqref{eq:M_matern} (which we established also holds here) completes the proof.

\subsection{Overview of Details for the SE Kernel}

For the SE kernel, we follow the same argument as the Mat\'ern kernel, but wherever the bump function from Lemma \ref{lem:simpler_function} is used, we replace it by the ``approximate bump'' function from Section \ref{sec:prelim}.  This creates a few more technical nuisances, but the argument is essentially the same, so we only outline the differences:
\begin{itemize}
    \item In the analog of Lemma \ref{lem:circ_func}, the function value is not exactly constant in the ``inner sphere'', and is not exactly zero outside the ``outer sphere'', but it is arbitrarily close (e.g., at least $0.99$ times the maximum in the former case, and below $0.01$ times the maximum in the former case).
    \item In specifying the $M$ functions, we no longer have disjoint supports, but we instead place the centers on a uniform grid, as was done in \cite{Sca17a,Bog18} and used in Section \ref{sec:prelim}.  Inside the ``plain'' sphere of radius $0.99\xi$ (see Figure \ref{fig:func_class_new_robust}), we can fit a cube of side-length $\frac{0.99\xi}{\sqrt d}$, and hence, the uniform grid still leads to $M$ of the form $M = \big( \frac{c'_0 \xi}{w} \big)^d$, albeit with a smaller constant $c'_0$ depending on $d$.
    \item Once the function class with the grid-like structure is established, instead of following the simplified steps for the Mat\'ern kernel in Section \ref{sec:simplified_matern}, we follow the slightly more involved (but still simple) steps from Section \ref{sec:ana_std_simple}.
\end{itemize}

\section{Further Auxiliary Lemmas} \label{app:further}

The following lemma states a well-known expression for the RKHS norm in terms of the Fourier transforms of the function and kernel.

\begin{lem} \emph{\citep[Sec.~1.5]{Aro50}} \label{lem:rkhs_norm}
    Consider an RKHS $\Hc$ for functions on $\RR^d$, corresponding to a kernel of the form $k(\xv,\xv') = k(r_{\xv,\xv'})$ with $r_{\xv,\xv'} = \xv - \xv'$, and let $K(\xi)$ be the $d$-dimensional Fourier transform of $k(\cdot)$.  Then for any $f \in \Hc$ with Fourier transform $F(\xi)$, we have
    \begin{equation}
        \|f\|_{\Hc} = \int \frac{ |F(\xi)|^2 }{ K(\xi) } d\xi.
    \end{equation}
    In addition, if $\Hc(D)$ is an RKHS on a compact subset $D \subseteq \RR^d$ with the same kernel as $\Hc$, then we have for any $f \in \Hc(D)$ that
    \begin{equation}
        \|f\|_{\Hc(D)} = \inf_{g} \|g\|_{\Hc(\RR^d)},
    \end{equation}
    where the infimum is over all functions $g \in \Hc(\RR^d)$ that agree with $f$ when restricted to $D$.
\end{lem}

While the following lemma is not used in this paper, it is stated because it is a key component of the analysis in \cite{Sca17a}, and can thus be contrasted with the key lemma of our analysis (Lemma \ref{lem:relating}).

\begin{lem} \emph{\cite{Aue95a}} \label{lem:auer}
    For any function $a(\yv)$ taking values in a bounded range $[0,A]$, we have
    \begin{align}
    \big| \EE_m[a(\yv)] - \EE_0[a(\yv)]\big| 
        &\le A\, d_{\rm TV}(P_0, P_m) \label{eq:auer_bound} \\
        &\le A\, \sqrt{  D(P_0 \| P_m) }, \label{eq:auer_bound2}
    \end{align}
    where $d_{\rm TV}(P_0,P_m) = \frac{1}{2} \int_{\RR^T} |P_0(\yv) - P_m(\yv)| \,{\rm d}\yv$ is the total variation distance.
\end{lem}

To simplify the final expression in Lemma \ref{lem:auer}, the divergence term therein can be further bounded using the following.

\begin{lem} \label{lem:div_bound}
    {\em \citep[Eq.~(44)]{Sca17a}}
    Under the definitions in Section \ref{sec:prelim}, we have
    \begin{equation}
    D(P_0 \| P_m) \le \sum_{j=1}^M \EE_0[N_j]\Dbar_m^j. \label{eq:div_bound}
    \end{equation}
\end{lem}

\end{document}